\documentclass[a4paper]{article}

\usepackage[english]{babel}
\usepackage[utf8x]{inputenc}
\usepackage[T1]{fontenc}
\usepackage{csquotes}

\usepackage[a4paper,top=3cm,bottom=2cm,left=3cm,right=3cm,marginparwidth=1.75cm]{geometry}

\usepackage{amsmath}
\usepackage{amssymb}
\usepackage{amsthm}
\usepackage{graphicx}
\usepackage[colorinlistoftodos]{todonotes}
\usepackage[colorlinks=true, allcolors=blue]{hyperref}
\usepackage{tikz}
\usetikzlibrary{arrows,automata}
\newtheorem{theorem}{Theorem}
\newtheorem{definition}{Definition}
\newtheorem{obs}{Observation}
\newtheorem{lemma}{Lemma}
\newtheorem{pro}{Proposition}
\newtheorem{example}{Example}

\newcommand{\Gurus}{\mathtt{Gu}} %ensemble de gurus d'une fonction de delegation
\newcommand{\GuruOf}{\mathtt{gu}}
\newcommand{\VoterSet}{\mathcal{N}}
\newcommand{\AbstSet}{\mathcal{A}}

\newcommand{\Acc}{\mathtt{Acc}} %ensemble de gurus acceptables d'un agent
\newcommand{\Score}{\mathtt{rk}}
\newcommand{\VotingPower}{\mathtt{vp}}
\newcommand{\At}[1]{\rho_{#1}}
\newcommand{\AP}{A_P}
\newcommand{\GPWA}{G_P^*}
\newcommand{\Gaux}{G^{\text{aux}}}%\overline{\mathcal{G}}} %graphe auxiliaire cas single peaked
\newcommand{\Vaux}{V^{\text{aux}}} %\overline{\mathcal{V}}}
\newcommand{\Aaux}{A^{\text{aux}}} %\overline{\mathcal{A}}}

\title{Iterative Delegations in Liquid Democracy with Restricted Preferences}
\author{Bruno Escoffier$^{1}$, Hugo Gilbert$^{2}$, Adèle Pass-Lanneau$^{1,3}$\\
$^1$ Sorbonne Universit\'e, CNRS,\\
LIP6 UMR 7606, 4 place Jussieu, 75005 Paris, France\\ 
\{bruno.escoffier,adele.pass-lanneau\}@lip6.fr\\
$^2$ Gran Sasso Science Institute, L'Aquila, Italy, hugo.gilbert@gssi.it\\
$^3$ EDF R\&D, 7 bd Gaspard Monge, 91120 Palaiseau, France
}

\begin{document}
\maketitle

\begin{abstract}
In this paper, we study liquid democracy,  a collective decision making paradigm which lies between direct and representative democracy. One main feature of liquid democracy is that voters can delegate their votes in a transitive manner %along a social network 
so that: A delegates to B and B delegates to C leads to A delegates to C. Unfortunately, this process may not converge as there may not even exist a stable state (also called equilibrium). %In fact, it is even NP-hard to know if such an equilibrium exists \cite{escoffier2018LD}.
In this paper, we investigate the stability of the delegation process in liquid democracy when voters have restricted types of preference on the agent representing them (e.g., single-peaked preferences). We show that various natural structures of preferences guarantee the existence of an equilibrium and we obtain both tractability and hardness results for the problem of computing several equilibria with some desirable properties. 
\end{abstract}

%\keywords{computational social choice; liquid democracy; algorithmic decision theory; delegative voting; games and equilibria; structured preferences}

\section{Introduction}
\label{sec:intro}
\emph{Interactive democracy} aims at using modern information technology, as Social Networks (SN), in order to make democracy more flexible, interactive and accurate \cite{brill2018interactive}. One of its most well known implementation is known as \emph{Liquid Democracy} (LD) \cite{green2015direct}. In a nutshell, LD allows voters to delegate transitively along an SN. More precisely, each voter may decide to vote directly or to delegate her vote to one of her neighbors, i.e., to use a \emph{representative}. In LD this representative can in turn delegate her vote and the votes that have been delegated to her to someone else. As a result, the delegations of the voters will flow along the SN until they reach a voter who decides to vote. This voter is called the \emph{guru} of the people she represents and has a voting weight equal to the number of people who directly or indirectly delegated to her. 
This approach is implemented in several online tools \cite{behrens2014principles,hardt2015google} and has been used by several political parties (e.g, the German Pirate party or the Sweden's Demoex party). One main advantage of this framework is its flexibility, as it enables voters to vote directly for issues on which they feel both concerned and expert and to delegate for others. 

On the other hand, one concern in LD is the stability of the induced delegation process \cite{bloembergen2018rational,escoffier2018LD}. Indeed, as the preferences of voters over their possible gurus can be conflicting, this process may end up in an unstable situation (i.e., a situation in which some voters would change their delegations). Unfortunately, Escoffier et al. \cite{escoffier2018LD} showed that an equilibrium of LD's delegation process may not exist, and that the existence of such an equilibrium is even NP-hard to decide. In this work, we obtain more positive results by considering several structures of preferences. We show that various natural structures of preferences guarantee the existence of an equilibrium and we obtain both tractability and hardness results for the problem of computing several equilibria with some desirable properties.  

%The stability of the delegation process is one of the several algorithmic issues raised by the liquid democracy setting. These issues have recently raised attention in the AI literature. In Section \ref{sec:RelatedWork}, we present related works that have recently tackled other issues of the liquid democracy setting. Moreover, we also give a short review of other works that are related to game-theoretic models induced by voting problems. In Section \ref{sec:notations}, we formally define our setting of liquid democracy, introduce our notations and specify the various questions that we address. Then, while Section~\ref{sec:generalPreferences} answers these questions when preferences of the voters over gurus are unrestricted, Sections~\ref{sec:singlePeaked},~\ref{sec:symmetrical} and \ref{sec:dbsn} answer these questions for several types of restricted preferences. Finally, Section~\ref{sec:conclusion} concludes and discuss several possible future works.  

\section{Related Work} \label{sec:RelatedWork}
Algorithmic issues of LD have recently raised attention in the AI literature. 
The idea underlying LD is that its flexibility should allow each voter to make an informed vote either by voting directly, or by finding a suitable guru. Several works have investigated this claim leading to both positive and negative results \cite{green2015direct,kahng2018liquid}.
On the one hand, Green-Armytage \cite{green2015direct}, proposed a spacial voting setting in which transitive delegations decrease on average a loss measure measuring how well the votes represent the voters. On the other hand, Kahng et al.~\cite{kahng2018liquid}, studied an election on a binary issue with a ground truth. In their model, no ``local'' procedure (i.e., a procedure working locally on the SN to organize delegations) can guarantee that LD is, at the same time, never less accurate and sometimes strictly more accurate than direct voting.

A possible pitfall of liquid democracy is that some agents may amass an enormous voting power. This issue was addressed by G\"olz et al. \cite{golzfluid} who considered the problem of minimizing the maximal weight of a guru given some delegations constraints. The authors designed a  $(1+\log(n))$-approximation algorithm (where $n$ is the number of voters) and show that approximating the problem within a factor $\frac{1}{2}\log_2(n)$ is NP-hard. Lastly, the authors gave evidence that allowing voters to specify multiple possible delegation options (instead of one) leads to a large decrease of the maximum voting power of a voter.

Christoff and Grossi \cite{christoffbinary} studied the potential loss of a rationality constraint when voters should vote on different issues that are logically linked and for which they delegate to different representatives. Following this work, Brill and Talmon~\cite{Brill2018liquid} considered an LD framework in which each voter should provide a linear order over possible candidates. To do so each voter may delegate different binary preference queries to different proxies. The delegation process may then yield incomplete and even intransitive preference orders. Notably, the authors showed that it is NP-hard to decide if an incomplete ballot obtained in this way can be completed to obtain complete and transitive preferences while respecting the constraints induced by the delegations.

Lastly, the stability of the delegation process of LD has been studied~\cite{bloembergen2018rational,escoffier2018LD}. Bloembergen et al. \cite{bloembergen2018rational} considered an LD setting where voters are connected in an SN and can only delegate to their neighbors in the network. The election is on a binary issue for which some voters should vote for the 0 answer and the others should vote for the 1 answer. Each voter only knows in a probabilistic way what is the correct choice for her, as well as for the others. This modeling leads to a  class of games, called \emph{delegation games} in which each voter aims at maximizing the accuracy of her vote/delegation. The authors proved the existence of pure Nash equilibria in several types of delegation games and gave upper and lower bounds on the price of anarchy, and the gain (i.e., the difference between the accuracy of the group after the delegation process and the one induced by direct voting) of such games. Following this line of research, Escoffier et al. \cite{escoffier2018LD} considered a more general type of delegation games in which voters have a preference order over who could be their guru, and each voter aims at being represented by the best possible one. The authors showed that an equilibrium may not exist in this type of delegation games. In fact, the existence of such an equilibrium is NP-hard to decide even if the SN is complete or is of bounded degree and it is W[1]-hard when parametrized by the treewidth of the SN. Then, the authors showed that such an equilibrium is guaranteed to exist, whatever the preferences of the voters, if and only if the SN is a tree and provided dynamic programming procedures to compute some equilibria with desirable properties.

In this work, we take an orthogonal approach to the one of Escoffier et al. \cite{escoffier2018LD}. Indeed, instead of investigating the effect of the properties of the SN on the equilibria of LD's delegation process, we put our emphasise on voters' preferences. More precisely, we assume that the SN is complete and we investigate several types of structured preferences that guarantee the existence of an equilibrium.

\section{Notations, Settings and Overview of the Results}
\label{sec:notations}
\subsection{Notations and Nash-Stable Delegation Functions}
Following the notations of Escoffier et al. \cite{escoffier2018LD}, we denote by \(\VoterSet =\{1,\ldots, n\}\) a set of voters that take part in a vote. In this former work, these voters were connected in a SN such that each voter could only delegate directly to their neighbors in the network. In this work, we assume  that any voter can delegate directly to any other voter. This is equivalent to considering a complete SN. Hence, each voter \(i\) can either vote herself, delegate to another voter \(j\), or abstain. 
We denote by \(d:\VoterSet \rightarrow \VoterSet \cup \{0\}\) a delegation function such that \(d(i) = i\) if voter \(i\) declares intention to vote, \(d(i) = j\) with \(j\in \VoterSet\setminus \{i\}\) if \(i\) delegates to \(j\), and \(d(i) = 0\) if \(i\) declares intention to abstain.
The \emph{set of gurus} \(\Gurus(d)\) is defined as the set of voters that vote given the delegations prescribed by \(d\), i.e., \(\Gurus(d) = \{i \in \VoterSet\ |\ d(i) = i\}\).
Delegations are transitive which means that if \(d(i) = j\), \(d(j) = k\), and \(d(k)=k\), then \(i\) is represented by \(k\). In the end, the voter who votes for \(i\), called the \emph{guru of \(i\)} and denoted by \(\GuruOf(i,d)\), is the voter in \(\Gurus(d) \cup \{ 0 \}\)  attained by following the delegations of the voters from \(i\). %Stated otherwise, $j = \GuruOf(i,d)$ if there exists a sequence of voters $i_1,\ldots, i_{\ell}$ such that $d(i_k) = i_{k+1}$, $i_1= i$, $i_{\ell} = j$ and $j \in \Gurus(d) \cup \{ 0 \}$. 
Note that the successive delegations starting from \(i\) may also end up in a circuit. %(i.e., $i=i_1$ delegates to $i_2$, who delegates to $i_3$, and so on up to $i_{\ell}$ who delegates to $i_k$ with $k\in\{1,\ldots,\ell-1\}$). 
In this case, we consider that all voters in the chain of delegations abstain, as none of them take the responsibility to vote. %More formally, in this situation, $\GuruOf(i_k,d) = 0$ for all $k\in\{1,\ldots,\ell\}$.\\

Each voter \(i\) has a preference order \(\succ_i\) over who could be their guru in \(\VoterSet\cup\{0\}\). 
For every voter \(i \in \VoterSet\), and for every \(j,k \in \VoterSet \cup \{0\}\) we have that \(j \succ_i k\) if \(i\) prefers to delegate to \(j\) (or to vote if \(j=i\), or to abstain if \(j=0\)) rather than to delegate to \(k\) (or to vote if \(k=i\), or to abstain if \(k=0\)). We say that voter \(i\) is an \emph{abstainer} in \(P\) if she prefers to abstain rather than to vote, i.e., if \(0 \succ_i i\); she is a \emph{non-abstainer} otherwise. The set of abstainers is denoted by \(\AbstSet\). 
The collection of these preference orders defines a \emph{preference profile} \(P = \{\succ_i\ |\ i \in \VoterSet\}\). 

For illustrative purposes, we now give an example of a preference profile. %\(P\) representing the preferences that voters have over possible gurus.

\begin{example}[Ex. 3.2 in \cite{escoffier2018LD}] \label{ex1} Consider the following preferences with $n$$=$$3$ voters:
\begin{align*}
1&:2\succ_1 1 \succ_1 3 \succ_1 0\\ 
2&:3\succ_2 2 \succ_2 1 \succ_2 0\\ 
3&:1\succ_3 3 \succ_3 2 \succ_3 0 
\end{align*}
In this example, each voter \(i\) prefers to delegate to \((i \mod 3) + 1\) rather than to vote directly and each voter prefers to vote rather than to abstain.
\end{example}

A delegation function \(d\) is said to be \emph{Nash-stable for voter \(i\)} if 
\[\GuruOf(i,d) \succ_i g\quad \forall g \in (\Gurus(d) \cup \{0,i\}) \setminus \{\GuruOf(i,d)\}.\]
A delegation function \(d\) is \emph{Nash-stable} if it is Nash-stable for every voter in \(\VoterSet\). A Nash-stable delegation function is also called an \emph{equilibrium} in the sequel. 
Unfortunately, as noted in \cite{escoffier2018LD} such an equilibrium may not exist as Example~\ref{ex1} admits no equilibrium. In fact, sets of gurus of Nash-stable delegation functions correspond to the kernels of a particular digraph as proven by Escoffier et al. \cite{escoffier2018LD} and as explained in the following subsection.

\subsection{Delegation-Acceptability Graph and Existence of Equilibria} \label{sec:generalPreferences}

Let $\Acc(i) = \{j \in N | j \succ_i i \text{ and } j \succ_i 0\}$ be the set of voters to which voter $i$ would rather delegate to than to abstain or vote directly. A necessary condition for a delegation function to be Nash-stable is that $\GuruOf(i,d)\in \Acc(i)$ for every voter $i$ who delegates to another voter. Otherwise, voter $i$ would change her delegation to abstain or vote directly. We refer to $\Acc(i)$ as the set of \emph{acceptable gurus} for $i$ in a Nash-stable situation.
%\MessageFromAdele{on peut peut-etre enlever cette remarque vu qu'on structure les pref ensuite (et qu'on a besoin de tout le profil pour la convergence):}
Note that since we are interested in equilibria, we do not require the whole preference profile as input. Namely, the preferences of agent $i$ below $0$ or $i$ in her preference list have no influence on equilibria. In the sequel, we may define a preference profile only by giving, for every agent $i$, if she is an abstainer or not, and her preference profile on $\Acc(i)$.

\begin{definition}[\cite{escoffier2018LD}]
The \emph{delegation-acceptability digraph} is the digraph $\GPWA = (\VoterSet \setminus \AbstSet, \AP)$ where $\AP = \{ (i,j) \ |\ j\in \Acc(i)\}$.
\end{definition}

Stated differently, there exists an arc from non-abstainer $i$ to non-abstainer $j$ if and only if $i$ accepts $j$ as a guru.

\begin{example}
\label{ex:exampleSinglePeaked}
Let us consider the preference profile $P$ on 4 voters \(\{ 1,2,3,4\}\) given on the left side of Figure~\ref{fig:GPsinglepeaked}. Its delegation-acceptability digraph \(\GPWA\) is represented on the right side of Figure~\ref{fig:GPsinglepeaked}.
%\begin{align*}
%1&: 2 \succ_1 1 \\ 
%2&: 3 \succ_2 4 \succ_2 2 \\ 
%3&: 2 \succ_3 1 \succ_3 3 \\
%4&: 3 \succ_4 4 \\ 
%\end{align*}
\begin{figure}[h] 
\begin{minipage}[c]{.3\linewidth}
\begin{align*}
1&: 2 \succ_1 1 \succ_1 3 \succ_1 4 \succ_1 0 \\ 
2&: 3 \succ_2 4 \succ_2 2 \succ_2 1 \succ_2 0 \\ 
3&: 2 \succ_3 1 \succ_3 3 \succ_3 4 \succ_3 0\\
4&: 3 \succ_4 4 \succ_4 2 \succ_4 1 \succ_4 0 
\end{align*}
\end{minipage}\hfill
\begin{minipage}[c]{.65\linewidth}
   \centering
    \scalebox{0.7}{\begin{tikzpicture}[->,>=stealth',shorten >=1pt,auto,node distance=3cm,semithick]

  \node[circle,draw,text=black] (1)   at (0,0)                 {1};
  \node[circle,draw,text=black] (2)   at (2,0)                 {2};
  \node[circle,draw,text=black] (3)   at (4,0)                 {3};
  \node[circle,draw,text=black] (4)   at (6,0)                 {4};
  
  \path (1) edge[->]  node {} (2)
        (2) edge [<->]  node {} (3)
  	    (4) edge[->]  node {} (3)
	    (2) edge[->, bend right=20]  node {} (4)
  	    (3) edge[->, bend right=20]  node {} (1);
	  \end{tikzpicture}}
    \caption{Delegation-acceptability digraph $\GPWA$.}
    \label{fig:GPsinglepeaked}
\end{minipage}
\end{figure}
\end{example}

%The delegation-acceptability digraph $\GPWA$ %and the delegation-acceptability digraph without abstainers $\GPWA$ 
%of $P$ is given in Figure \ref{delAcc}. 

%The main result of this section, stated in Proposition~\ref{th:equivGurusKernel}, is a characterization of all sets of gurus of equilibria, as specific subsets of vertices of the delegation-acceptability digraph. Let us introduce additional graph-theoretic definitions.
Given a digraph $G = (V, A)$, a subset of vertices $K \subseteq V$ is \emph{independent} if there is no arc between two vertices of $K$. It is \emph{absorbing} if for every vertex $u \notin K$, there exists $k \in K$ such that $(u,k) \in A$ (then we say that $k$ \emph{absorbs} $u$). A \emph{kernel} of $G$ is a subset of vertices that is both independent and absorbing. 

\begin{theorem}[\cite{escoffier2018LD}]
\label{th:equivGurusKernel}
Given a preference profile $P$ and a subset $K \subseteq \VoterSet \setminus \AbstSet$ of voters, there exists an equilibrium $d$ such that $\Gurus(d) = K$ if and only if $K$  is a kernel of $\GPWA$.
\end{theorem}

For instance, in Example \ref{ex:exampleSinglePeaked} the only kernel of \(\GPWA\) is \(\{1,4\}\) which corresponds to the equilibrium where voters 1 and 4 vote, voter 2 delegates to voter 4 and voter 3 delegates to voter 1.

%\MessageFromAdele{ajout de cette remarque:}
Note that given a set $K$ which is a kernel of $\GPWA$, it is straightforward to retrieve an equilibrium $d$ such that $\Gurus(d) = K$. E.g., the delegation function where every voter in $K$ votes, and every voter not in $K$ delegates to her preferred voter in $K$, is Nash-stable.  
%\MessageFromAdele{peut etre l'endroit pour se premunir contre la rqe du proxy voting}
%\MessageFromHugo{Proposition pour se premunir:}
Hence, surprisingly, given any equilibrium \(d\), there exists an equilibrium \(d'\) such that \(\GuruOf(i,d) = \GuruOf(i,d')\) for every voter \(i\) and where each voter delegates directly to her guru in $d'$. Note however that the transitivity property of delegations is crucial to our setting as it is at the root of the instability of the delegation process. %\MessageFromAdele{Je trouve ca bien !}

The problem of determining if an equilibrium exists is equivalent to the problem of determining if a digraph admits a kernel which is NP-complete \cite{chvatal1973computational}. 
 Interestingly, we will show in the sequel that for several natural structured preference profiles (e.g., single-peaked profiles) an equilibrium always exists. For these structured preference profiles we will investigate if we can compute  equilibria verifying particular desirable properties, and tackle convergence issues.

 %\MessageFromHugo{Peut etre juste: In fact the authors showed that the problem, called \textbf{EX}, of determining if an equilibrium exists is NP-complete.}

 %\MessageFromBruno{Ca me semble bizarre de mettre la section 4 (known results [...]). J'aurais plutot mis son contenu  ici, voire partiellement en 3.1.}

%\vspace{0.3cm}
%\noindent\fbox{\parbox{12cm}{
%\textbf{EX}\\
%\emph{INSTANCE}: A preference profile %\(P\).\\%describing the preferences of the voters in \(\VoterSet\) over who in \(\VoterSet \cup \{ 0 \}\) will represent them as guru.\\
%\emph{QUESTION:}  Does there exist an equilibrium?
%}}
%\vspace{0.5cm}

\subsection{Problems Investigated}

\subsubsection{Optimization and decision problems on equilibria.}

We investigate the same decision and optimization problems as in \cite{escoffier2018LD} which are all related to the equilibria of LD's delegation process.

Firstly, given a voter \(i\in \VoterSet\setminus\AbstSet\), problem \textbf{MEMB} aims at deciding if there exists an equilibrium for which \(i\) is a guru.
\begin{center}
\noindent\fbox{\parbox{12cm}{
\textbf{MEMB}\\
\emph{INSTANCE:} A preference profile \(P\) and a voter \(i \in \VoterSet \setminus \AbstSet\).\\% describing the preferences of the voters in \(\VoterSet\) over who in \(\VoterSet\cup\{ 0 \}\) will represent them as guru, and a voter \(i \in \VoterSet \setminus \AbstSet\).\\
\emph{QUESTION:}  Does there exist an equilibrium \(d\) for which \(i\in \Gurus(d)\)?
}}
\end{center}

Also, we will try to find equilibria that are optimal w.r.t. some objective functions. First, problem \textbf{MINDIS} tries to find an equilibrium that satisfies most the voters, where the \emph{dissatisfaction} of a voter \(i\in \VoterSet\) with respect to a delegation function \(d\)  is given by \(\Score(i,d)-1\) where \(\Score(i,d)\) is the rank of \(\GuruOf(i,d)\) in the preference profile of \(i\). Second, problem \textbf{MINMAXVP} tries to avoid that a guru would amass too much voting power, where the voting power \(\VotingPower(i,d)\) of a guru \(i\in \VoterSet\) w.r.t. a delegation function \(d\) is defined as \(\VotingPower(i,d) = |\{j\in \VoterSet| \GuruOf(j,d) = i\}|\). % of determining an equilibrium \(d\) minimizing \(\max_{i\in \Gurus(d)} \VotingPower(i,d)\).
Last, problem \textbf{MINABST} tries to determine an equilibrium \(d\) minimizing the number of people who abstain, i.e., \(|\{i\in \VoterSet | \GuruOf(i,d) = 0\}|\).
\begin{center}
\noindent\fbox{\parbox{12cm}{
Problems \textbf{MINDIS, MINMAXVP} and \textbf{MINABST}\\
\emph{INSTANCE:} A preference profile \(P\).\\ %describing the preferences of the voters in \(\VoterSet\) over who in \(\VoterSet\cup\{ 0 \}\) will represent them as guru.\\
\emph{SOLUTION:} A Nash-stable delegation function \(d\).\\
\emph{MEASURE for \textbf{MINDIS}:} \(\sum_{i\in \VoterSet} (\Score(i,d)-1)\) (to minimize).\\
\emph{MEASURE for \textbf{MINMAXVP}:} \(\max_{i\in \Gurus(d)} \VotingPower(i,d)\) (to minimize).\\
\emph{MEASURE for \textbf{MINABST}:} \(|\{i\in \VoterSet | \GuruOf(i,d) = 0\}|\) (to minimize).
}}
\end{center}
%\vspace{0.3cm}

%Note that \textbf{MINDIS}, \textbf{MINMAXVP} and \textbf{MINABST} require to determine if there exists a feasible solution. Hence they are at least as hard as \textbf{EX}.

%he last question that we investigate captures the dynamic nature of delegations. Indeed, it is not because a Nash-stable delegation function exists that an iterative sequence of delegations will necessarily reach an equilibrium.

\subsubsection{Convergence of Iterative Delegations.}

As we will focus on instances where an equilibrium always exists, a natural question is whether a dynamic delegation process (necessarily) converges. As classically done in game theory (see e.g.~\cite{NisanSVZ11}), we consider dynamics where iteratively one voter has the possibility to change her delegation/vote. 

In a dynamics, we are given a starting delegation function \(d_0\) and a token function \(T:\mathbb{N}^* \rightarrow \VoterSet\) which specifies that voter \(T(t)\) has the token at step \(t\): she has the right to change her delegation. This gives a sequence of delegation functions \((d_t)_{t\in \mathbb{N}}\) where for any \(t\in \mathbb{N}^*\), if \(j\neq T(t)\) then \(d_t(j)=d_{t-1}(j)\). A dynamics is said to converge if there is a \(t^*\) such that for all \(t\geq t^*\), \(d_t=d_{t^*}\). 
Given \(d_0\) and \(T\), a dynamics is called a {\it better response dynamics} or {\it Improved Response Dynamics} (IRD) if for all \(t\), \(T(t)\) chooses a move that strictly improves her outcome if any, otherwise does not change her delegation; it is called a {\it Best Response Dynamics} (BRD) if for all \(t\), \(T(t)\) chooses \(d_t(i)\) so as to maximize her outcome. Note that a BRD is also an IRD.
We will assume, as usual, that each voter has the token infinitely many times. A classical way of choosing such a function \(T\) is to consider a permutation \(\sigma\) over the voters in \(\VoterSet\), and to repeat this permutation over time to give the token (if \(t=r \mod n\) then \(T(t)=\sigma(r)\)). These dynamics are called {\it permutation dynamics}.

The last problems that we investigate, denoted by \textbf{IR-CONV} and \textbf{BR-CONV}, can be formalized as:
\begin{center}
\fbox{\parbox{12cm}{
\textbf{IR-CONV (resp.  BR-CONV)}\\
\emph{QUESTION:}  Does a dynamic delegation process under IRD (resp. BRD) necessarily converges whatever the preference profile \(P\) and token function \(T\).
}}
\end{center}

\subsection{Summary of the Results}

%The focus of the previous work of \cite{escoffier2018LD} was on the impact of the structure of the social network on the complexity of the mentioned problems, while preferences of voters remain very general. Our purpose is to investigate the problems under restricted preferences.
%We assume, in contrast with \cite{escoffier2018LD}, that the SN is complete, i.e., all voters may delegate to any other voter. Although this might seem a tough restriction of the model of this previous work, the problem of existence of equilibrium was already proven NP-complete in this case, under general preferences. This motivates the question of identifying preference structures under which equilibrium always exist.
%small world networks

As explained above, our purpose is to investigate the problems under restricted preferences.
In Section~\ref{sec:singlePeaked}, we study \emph{single-peaked preference profiles}, where agents are ordered on a line and they prefer gurus that are ``close'' to them on this axis. 
%\MessageFromBruno{For DB profiles: emphasize the fact that spatial models of preferences are classical (\cite{BOGOMOLNAIA200787}), where objects (here voters) are embedded in a $k$-dimensional space and a distance (usually, Euclidean) measures how far two objects are from eachother.
%Arguably, Single-peaked preferences and spatial preferences are two of the main/most studied structures of preferences.}
In Section~\ref{sec:symmetrical}, we investigate \emph{symmetrical preference profiles}, where all pair of voters always accept each other as guru, or reject each other.
Finally, as classically done in the framework of spatial preferences~\cite{BOGOMOLNAIA200787}, we consider in Section~\ref{sec:dbsn} that voters are embedded in a metric space. They accept as possible gurus voters that are close to them in this space. We denote these preference profiles as \emph{distance-based profiles}.

For each of these preference structures, we first show that an equilibrium always exists. Our results for problems \textbf{IR-CONV}, \textbf{BR-CONV}, \textbf{MEMB}, \textbf{MINDIS}, \textbf{MINMAXVP} and \textbf{MINABST} are presented in Table~\ref{tab:synthesis}.
%where the different lines of the table correspond to different structures of preferences.
Note that all missing proofs are deferred to the Appendices.

\begin{table}[!h]
\begin{center}
\scalebox{0.76}{
\begin{tabular}{|c|c|c|c|c|c|c|}
\hline
Preferences & \textbf{IR-CONV} & \textbf{BR-CONV}  & \textbf{MEMB} & \textbf{MINDIS} & \textbf{MINMAXVP} & \textbf{MINABST}\\
\hline
Single-Peaked &  Not Always & Not Always & $O(n^2)$ & $O(n^3)$ & $O(n^3)$ & $O(n^3)$\\
\hline
Symmetrical &  Not Always & Always & Always Exists & NP-Hard & NP-Hard & NP-Hard\\
\hline
Distance-Based &  Not Always & Always & NP-Complete & NP-Hard & NP-Hard & NP-Hard\\
\hline
\end{tabular}}
\end{center}
\caption{\label{tab:synthesis} Synthesis of Results.}
\end{table}

\section{Single-Peaked Preferences}
\label{sec:singlePeaked}

\subsection{Definition}
In this section, we consider that voters can be ordered on a line; this ordering \(<\) may represent, e.g., the political positions of the voters on a left-right ladder. We assume that voters are indexed w.r.t. this ordering and we identify them with their index in $\{1, \dots, n\}$.
A preference profile is \emph{single-peaked for voter $i\in \VoterSet$} if for every $j,k \in \VoterSet$, 
\begin{equation*}
 (i < j < k\text{ or }k < j < i) \implies j \succ_i k.\notag
\end{equation*}
A preference profile is \emph{single-peaked} if it is single-peaked for all voters. 
For instance, the preference profile given in Example~\ref{ex:exampleSinglePeaked} is single-peaked. 

In a Single-Peaked (SP) preference profile, if a voter delegates to a guru on her left (and similarly on her right), she prefers to delegate to the closest possible. Note that in $i$'s preference list, we allow $i$ (vote) and $0$ (abstention) to be in any position (differently from the traditional definition of single-peakedness).
It represents the fact that voter $i$ prefers to delegate to close gurus, but then, beyond a given threshold on her left (resp. right), she prefers to abstain or vote herself rather than to delegate to a guru that is too far from her opinions.

SP preferences are one of the most well-known restrictions of preferences in social choice theory. They where introduced by Black \cite{black1948rationale} who showed that they solve the Condorcet paradox in the sense that a weak Condorcet winner always exists with SP preferences.  Furthermore, SP electorates have many desirable properties: they induce a simple characterization of strategy proof voting schemes \cite{moulin1980strategy}; they are easily recognizable \cite{escoffier2008single}; and they often lead to more desirable complexity results (e.g., in multiwinner elections, where the goal of the election is to elect a committee representing best a set of voters \cite{betzler2013computation}).   

\subsection{Existence of Equilibrium}

We now establish that the existence of an equilibrium is guaranteed for an SP preference profile $P$.
A digraph $G$ is an \emph{interval catch digraph} \cite{prisner1994algorithms} with vertex-set $\VoterSet = \{1, \dots, n\}$ if for every $i \in \VoterSet$, there exists $l_i, r_i \in \VoterSet$ such that $l_i \leq i \leq r_i$ and the out-neighborhood of $i$ in $G$ is the subset $\{ l_i, \dots, r_i\} \backslash \{i\}$. These digraphs are naturally related to SP preference profiles by the following proposition.

%\begin{restatable}{pro}{PROPSPintervalcatch}
\begin{pro}
\label{prop:SPintervalcatch}
If $P$ is an SP preference profile, then its delegation-acceptability digraph $\GPWA$ is an interval catch digraph.
\end{pro}
%\end{restatable}

%\MessageFromAdele{Micro sketch, j'ai l'impression que c'est suffisant?}\MessageFromBruno{Oui je trouve aussi. Pas la peine de la preuve en annexe je pense.} \MessageFromHugo{Ok pour moi aussi sans preuve en annexe}
Indeed, note first that if we remove abstainers from profile $P$, the remaining profile is still single-peaked. Then by defining $l_i$ (resp. $r_i$) the smallest (resp. largest) voter that $i$ accepts as guru, it is easy to check that $\GPWA$ is an interval catch digraph.

Considering the preference profile from Example~\ref{ex:exampleSinglePeaked}, the digraph $\GPWA$ is clearly the interval catch digraph defined by the values $l_1 = 1$, $r_1 = 2$, $l_2= 2$, $r_2=4$, $l_3=1$, $r_3=3$, $l_4=3$ and $r_4=4$.

From the equivalence stated in Theorem~\ref{th:equivGurusKernel}, deciding the existence of an equilibrium is equivalent to deciding the existence of a kernel in the delegation-acceptability digraph. 
It was proven by Prisner \cite{prisner1994algorithms} that a kernel in an interval catch digraph always exists and is computable in $O(n^2)$ time. This leads to a polynomial algorithm for computing an equilibrium for an SP preference profile.  

\begin{theorem}\label{theo:stableSP}
An SP preference profile always admits an equilibrium. Furthermore, an equilibrium can be computed in $O(n^2)$.
\end{theorem}

%\subsection{Kernels in Interval Catch Digraphs}
\subsection{Equilibria and Optimization}

Theorem~\ref{theo:stableSP} addresses the question of computing one equilibrium. We now provide an additional characterization of sets of gurus of equilibria, %structural results on kernels in interval catch digraphs, 
that will be a convenient tool for solving other decision or optimization problems on equilibria. 

Let us define an auxiliary digraph $\Gaux= (\Vaux, \Aaux)$ associated with $\GPWA$ as follows. The vertex-set $\Vaux$ contains the set of voters $\{1, \dots, n\}$, plus a source $s$ and a sink $t$. 
For $i < j$, the arc-set $\Aaux$ contains the arc $(i,j)$  if the pair $\{i,j\}$ is a kernel of the subgraph of $\GPWA$ induced by $\{i, \dots,j \}$.
It contains the arc $(s,j)$ (resp. the arc $(i,t)$) if the singleton $\{j\}$ (resp. $\{i\}$) is a kernel of the subgraph of $\GPWA$ induced by $\{1, \dots, j\}$ (resp. $\{i, \dots, n \}$).
%$\mathcal{G}[\{ i, \dots, j\}]$ (resp. $\mathcal{G}[\{1, \dots, j\}]$, $\mathcal{G}[\{i, \dots, n\}]$).

For illustration purposes, the auxiliary digraph of the preference profile from Ex.~\ref{ex:exampleSinglePeaked} is given in Figure~\ref{fig:Gaux}. The two successors of source $s$ are $1$ and $2$: indeed the singletons $\{1\}$ and $\{2\}$ are kernels of the subgraph induced by $\{1\}$ and $\{1,2\}$ respectively. Vertices 3 or 4 do not absorb vertex 1, hence they are not successors of $s$. Between two vertices of $\{1, \dots, 4\}$ the only arc in $\Gaux$ is $(1,4)$, because all other pairs of vertices are neighbors, while $\{1,4\}$ is a kernel of $\GPWA$.

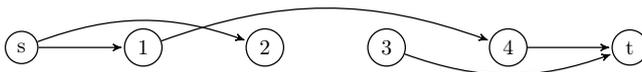
\begin{figure}[!h] 
   \centering
    \scalebox{0.8}{\begin{tikzpicture}[->,>=stealth',shorten >=1pt,auto,node distance=3cm,semithick]
  \node[circle,draw,text=black] (s)   at (-2,0)                 {s};
  \node[circle,draw,text=black] (1)   at (0,0)                 {1};
  \node[circle,draw,text=black] (2)   at (2,0)                 {2};
  \node[circle,draw,text=black] (3)   at (4,0)                 {3};
  \node[circle,draw,text=black] (4)   at (6,0)                 {4};
  \node[circle,draw,text=black] (t)   at (8,0)                 {t};
  
  \path (s) edge[->]  node {} (1)
        (s) edge [->, bend left=20]  node {} (2)
        (1) edge [->, bend left=20]  node {} (4)
	    (3) edge[->, bend right=20]  node {} (t)
  	    (4) edge[->]  node {} (t);
	  \end{tikzpicture}}
    \caption{Auxiliary digraph $\Gaux$ of $\GPWA$ for $P$ the preference profile of Example~\ref{ex:exampleSinglePeaked}.}
    \label{fig:Gaux}
\end{figure}

The importance of the auxiliary digraph is given by the following proposition.

%\begin{restatable}{pro}{PROCorrKernelPath}
\begin{pro}
\label{prop:corrKernelPath}
There is a one-to-one correspondence between sets of gurus of equilibria for preference profile $P$, and $s-t$ paths in the auxiliary digraph of $\GPWA$.
%kernels of $\GPWA$ and $s-t$ paths in its auxiliary digraph $\Gaux$.
\end{pro}
%\end{restatable}

%Note that the existence result from Prisner \cite{prisner1994algorithms} implies that there always exists at least one $s-t$ path in $\Gaux$.
%We apply this proposition by taking as input the delegation-acceptability digraph, which is an interval catch digraph, and building its auxiliary digraph $\Gaux$.

%With Proposition~\ref{prop:equivGurusKernel} it comes that the $s-t$ paths in the auxiliary digraph $\Gaux$ are exactly the sets of gurus of equilibria.

The proof of Proposition~\ref{prop:corrKernelPath} relies on a technical lemma on kernels of interval catch digraphs, which is stated and proven in Appendix. We obtain a one-to-one correspondence between sets of gurus of equilibria, kernels of $\GPWA$, and $s-t$ paths of the auxiliary digraph. 
Using this result, it is possible to solve problems \textbf{MEMB}, \textbf{MINDIS}, \textbf{MINMAXVP} and \textbf{MINABST} by transforming them into path problems in the auxiliary digraph. The results we obtain are given in the following theorem.

% We show that the auxiliary digraph is computable in \(O(n^2)\) (it is trivially computable in $O(n^3)$ time). Then we are able to solve problems \textbf{MEMB}, \textbf{MINDIS}, \textbf{MINMAXVP} and \textbf{MINABST} by transforming them into path problems in the auxiliary digraph. While the first one can then be solved in \(O(n^2)\), the other problems which require to label the edges of the auxiliary digraph can be solved in \(O(n^3)\). We sum up these results in the following Theorem.

%\begin{restatable}{theorem}{AUXDIGRAPH}
\begin{theorem}
\label{thrm:auxdigraph}
Given an SP preference profile $P$: the auxiliary digraph of $\GPWA$ is computable in $O(n^2)$ time; problem \textbf{MEMB} is solvable in $O(n^2)$ time; problems \textbf{MINDIS}, \textbf{MINMAXVP} and \textbf{MINABST} are solvable in $O(n^3)$ time.
\end{theorem}
%\end{restatable}

\begin{proof} (sketch) For problem \textbf{MINABST} we sketch the proof of the equivalence with a path problem in $\Gaux$ (full proof of the theorem is deferred to the Appendix).
Given $d$ an equilibrium, with Proposition~\ref{prop:corrKernelPath} the set $K= \Gurus(d)$ forms an $s$-$t$ path in $\Gaux$. We claim that to count the number of voters who abstain in $d$, it is sufficient to sum, for every pair of successive gurus $k$, $k'$, the number $a_{k,k'}$ of voters between $k$ and $k'$ who prefer abstention over $k$ and $k'$. Indeed because preferences are SP, any non-guru delegates to the closest guru on her left, or the closest guru on her right, or abstains. Thus $\Gaux$ can be labeled on arcs with the values $a_{i,j}$, and \textbf{MINABST} can be solved by finding a shortest $s$-$t$ path in $\Gaux$. %\qed
\end{proof}

\subsection{Convergence of Dynamics}

As Theorem~\ref{theo:stableSP} asserts that an equilibrium always exists in the SP case, it is worth considering convergence of dynamics in this setting. Unfortunately, such a convergence is not guaranteed. Indeed, Escoffier et al. \cite{escoffier2018LD} provide a BRD permutation dynamics
that does not converge for the preference profile of Example~\ref{ex:exampleSinglePeaked}. As this profile is SP, convergence of BRDs are not guaranteed for SP preferences. %Note that it starts at $d_0$ where all voters declare intention to vote.

\section{Symmetrical Preference Profiles}\label{sec:symmetrical}

\subsection{Definition, Existence of Equilibria and Membership Problem}
We consider in this section the case where the preferences are symmetrical in the sense that $i\in \Acc(j)$ if and only if $j\in \Acc(i)$. As we will see in Section~\ref{sec:dbsn}, this is a particular case of the more general distance-based preference profiles. In the case of symmetrical preference profiles, the delegation-acceptability digraph has the arc $(i,j)$ iff it has the arc $(j,i)$ (it is symmetrical). Then, as noted in~\cite{escoffier2018LD}, the existence of an equilibrium is trivially guaranteed (take any maximal independent set of $\GPWA$), and for any non-abstainer $i$ there exists an equilibrium in which $i$ is a guru (take a maximal independent set containing $i$). %In other words, the answer to {\bf MEMB} is always yes. %More generally, given a set of voters, deciding if there exists  an equilibrium in which every voter in this set is a guru is easy: we just have to check whether the set is independent or not.  

\subsection{Equilibria and Optimization}

Though the existence of equilibrium is trivial in the case of symmetrical preference profiles, we now show that {\bf MINDIS}, {\bf MINMAXVP} and {\bf MINABST} are computationally hard, in contrast with the results of the SP case. These results, as well as another hardness result in Section~\ref{subsec:dbsnmembership}, are all based on a reduction from the \textbf{3-Satisfiability} (3-SAT) problem, known to be NP-complete~\cite{gj}, and use the same gadget digraph that we present now.\\

%Let us first define the problem 3-SAT, known to be NP-complete~\cite{gj}:\\

\noindent\fbox{\parbox{12cm}{
\textbf{3-SATISFIABILITY (abbreviated by 3-SAT)}\\
\emph{INSTANCE:} A set $U$ of $n_u$ binary variables, a collection $C$ of $n_c$ disjunctive clauses of 3 literals, where a literal is a variable or a negated variable in $U$.\\
\emph{QUESTION}: Is there a truth assignment for $U$ that satisfies all clauses in $C$?
}}

\vspace{0.5cm}

To an instance $(U,C)$ of 3-SAT we associate the symmetrical digraph $G_{U,C}$ defined as follows:
\begin{itemize}%[topsep=0pt]
	\item For each variable $x_i \in U$, we create two adjacent vertices $v^x_{it}$ and $v^x_{if}$, called variables vertices, representing respectively the literals $x_i$ and $\overline{x_i}$. 
    \item For each clause $c_j \in C$ we create one vertex $v^c_{j}$, called clause vertex; $v^c_{j}$ is adjacent to the three vertices corresponding to the three literals in $c_j$. 
\end{itemize}

To illustrate this construction, consider the following 3-SAT instance: 
\begin{align}
U &= \{x_1,x_2,x_3,x_4,x_5\} \label{form1a}\\
C &= \{ (x_1 \lor x_2 \lor \lnot x_3), (\lnot x_2 \lor \lnot x_4\lor x_1), (\lnot x_1 \lor x_3\lor x_5)\} \label{form1b}
\end{align}
Figure~\ref{graphguc} gives the corresponding digraph $G_{U,C}$.

\begin{figure}[!h] 
   \centering
    \scalebox{0.8}{\begin{tikzpicture}[<->,auto,node distance=3cm,semithick]
  \node[circle,draw,text=black] (AT)   at (-4,1.5)                 {$v^x_{1t}$};
  \node[circle,draw,text=black] (BT)   at (-2,1.5)                 {$v^x_{2t}$};
  \node[circle,draw,text=black] (CT)   at (0,1.5)                 {$v^x_{3t}$};
  \node[circle,draw,text=black] (DT)   at (2,1.5)                {$v^x_{4t}$};
  \node[circle,draw,text=black] (ET)   at (4,1.5)                {$v^x_{5t}$};
  \node[circle,draw,text=black] (AF)   at (-4,3)                 {$v^x_{1f}$};
  \node[circle,draw,text=black] (BF)   at (-2,3)                 {$v^x_{2f}$};
  \node[circle,draw,text=black] (CF)   at (0, 3)                 {$v^x_{3f}$};
  \node[circle,draw,text=black] (DF)   at (2, 3)                {$v^x_{4f}$};
  \node[circle,draw,text=black] (EF)   at (4, 3)                {$v^x_{5f}$};
  \node[circle,draw,text=black] (C1)   at (-2,0)                 {$v^c_1$};
  \node[circle,draw,text=black] (C2)   at (0,0)                 {$v^c_2$};
  \node[circle,draw,text=black] (C3)   at (2,0)                 {$v^c_3$};
    
  \path
  	    (C1) edge  node {} (AT)
        (C1) edge  node {} (BT)
        (C1) edge  node {} (CF)

  	    (C2) edge  node {} (BF)
        (C2) edge  node {} (DF)
        (C2) edge  node {} (AT)

  	    (C3) edge  node {} (AF)
        (C3) edge  node {} (CT)
        (C3) edge  node {} (ET)

        (AT) edge node {} (AF)
        (BT) edge node {} (BF)
        (CT) edge node {} (CF)
        (DT) edge node {} (DF)
        (ET) edge node {} (EF);
	  \end{tikzpicture}}
    \caption{Gadget digraph $G_{U,C}$}
    \label{graphguc}
\end{figure}
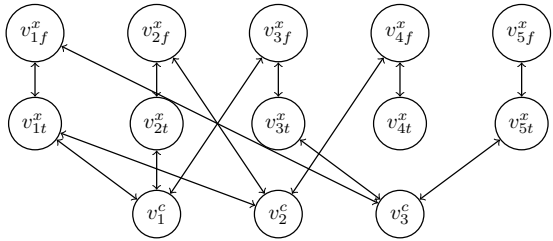

%\begin{restatable*}{obs}{OBSGUC}\label{obsguc}
\begin{obs}\label{obsguc}
$G_{U,C}$ has a kernel containing no clause vertex if and only if $(U,C)$ is satisfiable.
\end{obs}
%\end{restatable*}

From this construction we derive the following hardness results.

\begin{theorem}\label{th:hardnesssymmetric}
Given a symmetrical preference profile $P$:\\ 
-- it is NP-hard to decide whether there exists an equilibrium where no voter abstains, or not. Thus, in particular, {\bf MINABST} is NP-hard.\\
-- {\bf MINDIS} is NP-hard  even if there are no abstainers.\\
-- {\bf MINMAXVP} is NP-hard  even if there are no abstainers.
\end{theorem}

\begin{proof}
	We only prove the first item (see the Appendix for the two other items), which directly follows from Observation~\ref{obsguc}. Let us consider a 3-SAT instance with a set $U$ of variables and a set $C$ of clauses. We create a preference profile with $2n_u$ voters  $v^x_{it}$ and $v^x_{if}$, $i=1,\ldots,n_u$, and $n_c$ voters $v^c_j,j=1,\dots,n_c$. A voter $v^c_j$ accepts to delegate to the 3 voters corresponding to the three literals in the clause (and they accept her by symmetry), and then $v^c_j$ prefers to abstain.
  Moreover,  $v^x_{it}$ and $v^x_{if}$ also accept to delegate to each other. Then they prefer to vote. Then an equilibrium where nobody (no voter $v^c_j$) abstains corresponds to a kernel in $G_{U,C}$ with no clause vertex. The result follows from Observation~\ref{obsguc}. %\qed
\end{proof}

\subsection{Convergence of Dynamics}

We now focus on the question of convergence under BRD in the case of symmetrical preference profiles. Interestingly, while there might be cycles in the SP case, we show now that under BRD the convergence is guaranteed under symmetrical profiles, and that this convergence occurs within a small number of steps.

Given a dynamics with token function $T$, let us define \emph{rounds} as follows:
\begin{itemize}%[topsep = 0pt]
	\item The first round is $[1,r_1]$ where $r_1$ is the smallest $r$ such that each voter receives the token at least once in  $[1,r]$.
    \item The $k^{st}$ round is $[r_{k-1}+1,r_k]$ where $r_{k}$ is the smallest $r$ such that each voter receives the token at least once in  $[r_{k-1}+1,r]$.
\end{itemize}
For instance, in the case of permutation dynamics, we have $r_k=kn$.
%\MessageFromBruno{Restatable does not seem to work here (no theorem name in the pdf).}
%\begin{restatable*}{theorem}{THCVBRD} \label{th:cvbrd}
\begin{theorem}\label{th:cvbrd}
Given a symmetrical preference profile $P$, a BRD dynamics always converges in at most 3 rounds.
\end{theorem}
%\end{restatable*}
Intuitively, one can show that symmetry implies that when a voter decides to vote she will not change her mind later. Then after two rounds the set of gurus is fixed, and in the third round each non-guru chooses her best guru, leading to a Nash equilibrium (see Appendix).

We now show that convergence is {\it not} guaranteed under better response dynamics, thus providing a notable difference between the two dynamics. This holds even if we start from the delegation $d_0$ where all voters declare intention to vote, as shown by the following example.

\begin{example}\label{ex:nonconvergencedbsn}
Let us consider the case of 4 voters, where $\Acc(1)=\Acc(3)=\{2,4\}$, $\Acc(2)=\Acc(4)=\{1,3\}$. They all prefer to vote than to abstain.

We give the token to $1,2,1,3,2,4,3,1,4\ldots$. Then the following is compatible with better response:
$d_1(1)=2$, $d_2(2)=3$, $d_3(1)=1$, $d_4(3)=4$, $d_5(2)=2$, $d_6(4)=1$, $d_7(3)=3$, $d_8(1)=2$, $d_9(4)=4$. At this point $d_9=d_2$, so this is a cycle. Intuitively, each voter $i$ delegates  to its neighbor $i+1$ (modulo 4); in the following step $i+1$ delegates to $i+2$, then we give the token back to $i$ who is no more happy with her delegation and decides to vote herself.  
\end{example}

\section{Distance-Based Preference Profiles}\label{sec:dbsn}

\subsection{Definition and Existence of Equilibria}

In this section, we restrict our attention to another type of structured preference profiles. We assume that to each pair $(i, j)$ of voters is associated a  distance $dist(i,j)=dist(j,i)\in \mathbb{R}_+$. %measuring how far [...].
Then, each voter $i$ has its own {\it acceptability threshold} $\At{i} \in \mathbb{R}_+$: she accepts as possible gurus the other voters that are at distance at most $\At{i}$ from her: 
$$\forall j \in \VoterSet\backslash \{i\},\quad j \in \Acc(i) \Leftrightarrow dist(i,j)\leq \At{i}$$
We say that such a preference profile is DB (Distance Based). An example is given in Appendix \ref{app:db} for illustration purposes (see Example \ref{ex:DB}).

Note that DB preference profiles may represent the case where voters are embedded in a metric space, as in spatial models of preferences~\cite{BOGOMOLNAIA200787}. They might be points in $\mathbb{R}^k$ (as in the example in appendix); they may also represent vertices of a given graph, the distance being the shortest path between vertices in the graph.\\ %\MessageFromBruno{parallel with social network? Maybe not!}\\

Any symmetrical preference profile is DB: indeed, consider the distance defined by $dist(i,j)=1$ if and only if $j\in \Acc(i)$ (or, equivalently for a symmetrical profile, $i\in \Acc(j)$), and $dist(i,j)=2$ otherwise, and set $\At{i}=1$ for any voter $i$. 
From this observation we immediately know that {\bf MINDIS}, {\bf MINMAXVP} and {\bf MINABST} are NP-hard in the case of DB preference profiles.\\

We now show that the existence of equilibrium, which was trivially guaranteed in the symmetrical case, is also guaranteed in this more general case.

\begin{theorem}
If preferences are DB, a preference profile always admits an equilibrium. Furthermore, an equilibrium can be computed in $O(n^2)$.
\end{theorem}
\begin{proof}

We give a $O(n^2)$ procedure that builds an equilibrium for any DB preference profile.
Build a set $K$ of voters by using the following recursive procedure. Let $S = \VoterSet \setminus \AbstSet$. Then, while $S$ is not empty, add to $K$ the voter $i$ of $S$ with smallest $\At{i}$ value and remove $i$ from $S$ as well as all voters accepting $i$ as guru. At the end of this process, $K = \{i_1,\ldots, i_m\}$ is a kernel of $\GPWA$. It is absorbing because each voter in $ \VoterSet \setminus K$ has at some point been removed from $S$ because it was absorbed by one element of $K$. It is also independent. Indeed, let us assume by contradiction that $i_k$ accepts to delegate to $i_l$ with $i_k,i_l \in K$. Then, necessarily $i_k$ has been added to $K$ before $i_l$. Otherwise, $i_k$ would have been removed from $S$ at the same time as $i_l$ and would not have been added to $K$. Hence, $\At{i_k} \leq \At{i_l}$ and $i_l$ accepts to delegate to $i_k$ which is not possible by the same argument. This procedure builds $K$ in $O(n^2)$ and the equilibrium induced by $K$ can easily be build in $O(n^2)$. %\qed
\end{proof}

We note that the proof does not rely on the fact that $dist$ is a distance: an equilibrium always exists as soon as $dist(i,j)=dist(j,i)$, even if the triangle inequality does not hold for instance.

\subsection{Membership Problem}\label{subsec:dbsnmembership}

Given that an equilibrium always exist, we now focus on the problem {\bf MEMB}. In the case of symmetrical preferences, any voter could be a guru. We show here a drastic difference in the case of DB preferences, as  {\bf MEMB} becomes NP-hard.

\begin{theorem}\label{th:membharddb}
{\bf MEMB} is NP-hard in the case of DB preference profiles, even if there are no abstainers.
\end{theorem}
\begin{proof} (sketch)
We only give the way the reduction is built (see Appendix for the full proof). Let us consider a 3-SAT instance with a set $U$ of variables and a set $C$ of clauses. We consider a graph made of:
\begin{itemize}%[topsep = 0 pt]
	\item The undirected version of the graph $G_{U,C}$ associated to $(U,C)$ (see Figure~\ref{graphguc} in Section~\ref{sec:symmetrical});
    \item Two adjacent vertices $v_t$ and $v_q$; $v_t$ is also adjacent to all clause vertices $v^c_j$.
\end{itemize}
Each vertex of this graph is a voter (we have $2n_u+n_c+2$ voters), and the distance between voters $i$ and $j$ is the shortest path (number of edges, the graph is unweighted) between the two vertices representing $i$ and $j$ in the graph. The acceptability threshold is 1 for all voters except $v_q$ which has an acceptability threshold of 2; they all prefer to vote than to abstain. We can show that the 3-SAT instance is satisfiable iff the DB profile induced by the corresponding distance admits a Nash-stable delegation function in which $v_q$ is a guru.  %\qed
\end{proof}

\subsection{Convergence of Dynamics}

Since an equilibrium always exists, we consider now the convergence of dynamics. Example~\ref{ex:nonconvergencedbsn} shows that, under better response, the convergence is not guaranteed in the case of symmetrical preference profiles. Therefore, it is the same in the case of DB preference profiles. We now extend Theorem~\ref{th:cvbrd} and show that under BRD the convergence is guaranteed under DB preference profiles.

\begin{theorem}~\label{th:bdrconvdb}
Given a DB preference profile, a BRD dynamics always converges.
\end{theorem}

\begin{proof} (sketch)
The full proof is given in Appendix. Let us recall that each voter has the token infinitely many times. Consider a DB preference profile $P$, and a BRD dynamics with a starting delegation $d_0$ and a token function $T$. We assume that voters are numbered $1,2\dots,n$ in such a way that $\At{i}\leq \At{i+1}$, $i=1,\dots,n-1$. 

Let us define $G$ as the set of voters which are gurus (vote) infinitely many times in the dynamics: $G=\{i_1,\dots,i_s\}$ with $i_1\leq i_2\leq \dots \leq i_s$. Note that, obviously, $G$ contains no abstainers. Since voters in $\VoterSet \setminus G$ are gurus finitely many times, let us consider a step $t_0$ such that, for any $t\geq t_0$, no voter in $\VoterSet \setminus G$ are gurus (they always delegate or abstain). 
Let $t_1$ be the first time $t>t_0$ such that $i_1$ has the token and decides to vote. Since $i_1$ decides to vote at $t_1$, no voter in $\Acc(i_1)$ is a guru. Then, while $i_1$ is a guru: 
\begin{itemize}%[topsep=0pt]
\item no voter $j>i_1$ in $\Acc(i_1)$ ever becomes a guru: indeed, since $\At{i}$ are in non decreasing order, if $j\in \Acc(i_1)$ then $i_1\in \Acc(j)$. While $i_1$ is a guru $j$ does not decide to vote.
\item no voter $j<i_1$ in $\Acc(i_1)$ ever becomes a guru: indeed, these are in $\VoterSet \setminus G$ and since $t_1\geq t_0$ we know that they always delegate or abstain.
\end{itemize}
Then no voter in $\Acc(i_1)$ becomes a guru, so $i_1$ will vote (be a guru) forever. By recursively defining $t_k$ as the first time $t>t_{k-1}$ such that $i_k$ has the token and decides to vote, we can show using similar arguments that $i_k$ remains a guru forever after time $t_k$. Thus, at time $t_{s}$: voters in $G$ are gurus forever, and voters in $\VoterSet \setminus G$ never become gurus. From $t_s$ we only have to wait for another round to reach a Nash-stable delegation function. %\qed
\end{proof}

We note that, as in the proof of existence of equilibrium, we made here no specific assumption on the function $dist$ except that $dist(i,j)=dist(j,i)$.

\section{Conclusion and Future Work}
\label{sec:conclusion}
We have investigated the stability of the delegation process in liquid democracy when voters have restricted types of preference on the agent representing them. Interestingly, while the existence of an equilibrium of this process is NP-hard to decide when preferences are unrestricted \cite{escoffier2018LD}, we have showed that various natural structures of preference, namely single-peaked, symmetrical and distance-based preferences, guarantee the existence of an equilibrium. For these structures of preference, we have obtained positive and negative results which surprisingly differ for the different structures of preference studied. For instance, while single-peaked preferences are the only ones studied that make it possible to solve efficiently all the optimization problems that we investigated, they also form the only type of restricted preferences studied that do not guarantee the convergence of the delegation process under best response dynamics.

For future work, we could extend our results to other structures of preferences. Secondly, it would be interesting to study the \emph{price of anarchy} of the delegation games induced by liquid democracy both under unrestricted and structured preferences.  
A last direction would be to extend our results to the framework of \emph{viscous democracy} \cite{boldi2011viscous}. In this setting, the weight of a delegation decreases exponentially with the length of the delegation path. Hence, the preferences of an agent would be defined on the delegation paths.

\bibliographystyle{alpha}
\bibliography{sample}

\section{Appendix of Section \ref{sec:singlePeaked}}

%\begin{comment}
%\textbf{Proposition \ref{prop:SPintervalcatch}. }{If $P$ is an SP preference profile, then its delegation-acceptability digraph $\GPWA$ is an interval catch digraph.}
%\MessageFromAdele{preuve necessaire ?}
%\begin{proof}
%Consider an SP preference profile $P$. Given $i \in \VoterSet\setminus \AbstSet$, let $l_i$ (resp. $r_i$) be the minimal (resp. maximal) agent (or $i$ at worst) that $i$ accepts to delegate to in \(\VoterSet\setminus \AbstSet\). %$l_i = \min (\Acc(i)\cup\{i\})\setminus \AbstSet$ (resp. $r_i = \max (\Acc(i)\cup\{i\})\VoterSet\setminus \AbstSet$).
%Then all agents $j < l_i$ are unacceptable as gurus for $i$. Single-peakedness also implies that $i$ accepts as a guru every agent that is between $l_i$ and $i$. A similar argument holds for $r_i$. Hence, in the delegation-acceptability digraph $\GP$, the out-neighborhood of $i$ is exactly $\{ l_i, \dots, r_i\} \backslash \{i\}$. \qed
%\end{proof}
%\bigbreak
%\end{comment}

\subsection{Proof of Proposition~\ref{prop:corrKernelPath}}
%In order to prove Proposition~\ref{prop:corrKernelPath},
We first give a technical lemma on the structure of kernels of interval catch graphs.
Consider an interval catch digraph $G$$=$$(V,A)$ defined by the values $l_i, r_i$, $i \in V$. For $U \subseteq V$, let $G[U]$ denote the subgraph of $G$ induced by the vertices in $U$.

\begin{lemma}
\label{lem:kernelIntervalCatch}
Let $K \subseteq V$ and $k^1, \dots, k^p$ be the vertices of $K$ in increasing order.
Define $I_0 = \{1, \dots, k^1\}$, $I_t = \{k^t, \dots, k^{t+1}\}$ for every $t \in \{ 1, \dots, p-1\}$, and $I_p = \{ k^p, \dots, n\}$.

Then the set $K$ is a kernel of $G$ if and only if $K \cap I_t$ is a kernel of $G[I_t]$ for every $t \in \{0, \dots, p\}$.
\end{lemma}

%Note that for every $t$, the set $K \cap I_t$ is a singleton or a pair of vertices. The idea of the lemma is that $K$ being a kernel can be expressed by conditions on pairs of successive vertices of $K$. 
%\MessageFromHugo{Dans cette preuve on melange un peu les differentes notions.}
%\MessageFromAdele{Repris}

\begin{proof}
Assume $K$ is a kernel of $G$. Clearly $K \cap I_t$ is an independent set of $G[I_t]$ for every $t \in \{0, \dots, p\}$. Assume there exists $t \in \{1, \dots, p-1 \}$ such that $K \cap I_t$ is not absorbing in $G[I_t]$, i.e., there exists $j$ such that $k^t < j < k^{t+1}$ and neither $k^t$ nor $k^{t+1}$ are out-neighbors of $j$. It implies $k^t < l_j \leq r_j < k^{t+1}$, and since $k^t$ and $k^{t+1}$ are successive vertices of $K$ it comes that no vertex of $K$ absorbs $j$, a contradiction. Hence $K \cap I_t$ must absorb all vertices in $G[I_t]$. Similar argument applies for $t=0$ and $t=p$. Hence $K \cap I_t$ is a kernel of $G[I_t]$ for every $t \in \{0, \dots, p\}$.

Conversely, assume $K \cap I_t$ is a kernel of $G[I_t]$ for every $t \in \{0, \dots, p\}$. Then $K$ is clearly an absorbing set of $G$. Assume it is not independent: there exist two vertices $k^{t}$ and $k^{t'}$ in $K$ that are neighbors. Assume w.l.o.g. $t < t'$ and $(k^t,k^{t'}) \in A$ (this can be assumed up to reversing the ordering of vertices). Then $k^{t'}$ is an out-neighbor of $k^t$, which implies $k^{t'} \leq r_{k^{t}}$. Since $t < t'$ we have also $k^{t+1} \leq k^{t'}$. Hence $k^{t+1} \leq r_{k^t}$ and $(k^t,k^{t +1}) \in A$, and $K \cap I_t$ is not an independent set in $G[I_t]$.%\qed
\end{proof}

We now provide the proof of Proposition~\ref{prop:corrKernelPath}.
\medbreak
\noindent\textbf{Proposition \ref{prop:corrKernelPath}. }{There is a one-to-one correspondence between sets of gurus of equilibria for the preference profile $P$, and $s-t$ paths in the auxiliary digraph of $\GPWA$.}

\begin{proof}
We prove the one-to-one correspondence between kernels of $\GPWA$ and $s-t$ paths of $\Gaux$. The result then follows with Theorem~\ref{th:equivGurusKernel}.
Since $\GPWA$ is an interval catch digraph, with Lemma~\ref{lem:kernelIntervalCatch}, a set \(K\) is a kernel of \(\GPWA\) is a kernel iff: every pair $\{k,k'\}$ of successive voters in $K$ (resp. $k$ the smallest element of $K$, $k'$ the largest element of $K$) is a kernel of the subgraph of $\GPWA$ induced by $\{k, \dots, k'\}$ (resp. $\{1, \dots, k\}$, $\{k', \dots, n\}$). By definition of \(\Gaux\), it is equivalent with \(K\) being the set of vertices of a \(s-t\) path in \(\Gaux\).
%\qed
\end{proof}

\subsection{Proof of Theorem \ref{thrm:auxdigraph}}
\noindent\textbf{Theorem \ref{thrm:auxdigraph}. }{
Given an SP preference profile $P$: the auxiliary digraph of $\GPWA$ is computable in $O(n^2)$ time; problem \textbf{MEMB} is solvable in $O(n^2)$ time; problems \textbf{MINDIS}, \textbf{MINMAXVP} and \textbf{MINABST} are solvable in $O(n^3)$ time.}

\medbreak
For the sake of readibility, we decompose this Theorem in Lemmas \ref{th:sp1} to \ref{th:sp5}.

\begin{lemma} \label{th:sp1}
Given an SP preference profile \(P\), the auxiliary digraph \(\Gaux\) of \(\GPWA\) is computable in $O(n^2)$ time.
\end{lemma}
\begin{proof}
The digraph $\GPWA$ is computable in $O(n^2)$ time.
Then we prove that it is possible to compute in linear time the out-neighborhood of every vertex of the auxiliary digraph $\Gaux$ of $\GPWA$.
Let $i \in \VoterSet$. Define for every $j > i$ the value $r^*_j = \min \{ r_k\ |\ k \in \{ i+1, \dots, j-1\},\ i \notin \Acc(k)\}$. We claim that: \emph{the pair $\{i,j\}$ is an absorbing set of $G_P^*[\{ i,  \dots, j\}]$ if and only if $j \leq r^*_j$}. Indeed, if $j \leq r^*_j$, then for every $k \in \{i+1, \dots,j-1\}$ either $i \in \Acc(k)$ or $j \leq r^*_j \leq r_k$: then $j \in \Acc(k)$. Hence in both cases the agent $k$ accepts $i$ or $j$ as a guru: $\{i,j\}$ is absorbing. Conversely, if $j > r^*_j$ then by definition of $r^*_j$ there exists $k \in \{i+1, \dots, j-1\}$ such that $i \notin \Acc(k)$ and $j > r_k$. Then $j \notin \Acc(k)$, and the vertex $k$ is neither absorbed by $i$ nor by $j$.

Using this claim, we prove that the out-neighborhood of $i$ can be computed in linear time by the following procedure. Initialize $j:=i+1$ and $r^* := +\infty$. While $j \leq n$, apply the following: (i) If $i \notin \Acc(j)$, $j \notin \Acc(i)$, and $j \leq r^*$, then add $(i,j)$ to $\Aaux$. (ii) If $i \notin \Acc(j)$, update $r^*$ by $r^* := \min\{ r^*, r_j\}$. (iii) Increment $j$.%\qed
\end{proof}

\begin{lemma}\label{th:sp2}
Given an SP preference profile, problem \textbf{MEMB} is solvable in \(O(n^2)\) time.
\end{lemma}
\begin{proof}
From Proposition~\ref{prop:corrKernelPath} the problem \textbf{MEMB} is equivalent to the problem of deciding the existence of an $s-t$ path in $\Gaux$ that goes through voter $i$, i.e., finding an $s-i$ path and a $i-t$ path. It can be done in $O(n^2)$ time by computing $\Gaux$ then performing two graph searches.%\qed
\end{proof}
%This result can be extended to the more general problem of deciding if a subset of voters can be part of the set of gurus of a Nash-stable delegation function.\\

\emph{Optimization problems.}
The optimization problems introduced in Section~\ref{sec:intro} use objective functions that depend not only on the set of gurus $\Gurus(d)$ but also on the values of $\GuruOf(i,d)$ for each voter $i$.

\begin{obs}
\label{obs:affectGuruSP}
Let $d$ be an equilibrium and let $K = \Gurus(d)$. Then the guru $\GuruOf(i,d)$ of every agent $i \notin K$ is the guru that $i$ prefers among: the closest voter of $K$ on her left, the closest voter of $K$ on her right; and abstention.
\end{obs}

\begin{proof}
Nash-stability implies that $\GuruOf(i,d)$ is the preferred choice of $i$ in $K \cup \{0,i\}$ (this holds for all preferences and not only SP preferences).
Let $i \notin K$. Since $K$ is the set of gurus of the equilibrium $d$, it comes $\GuruOf(i,d) \not= i$, hence $\GuruOf(i,d)$ is the element that $i$ prefers in $K \cup \{0\}$. By single-peakness, the preferred guru of $i$ in the set $K$ is either the closest guru on her left or the closest guru on her right. The result follows.%\qed
\end{proof}

A consequence is the following. Assume $d$ is an equilibrium with set of gurus $\Gurus(d)$, associated with an $s-t$ path of $\Gaux$. Let $(k,k')$ be an arc of this path. Then from Observation~\ref{obs:affectGuruSP} every $i \in \{k+1, \dots, k'-1\}$ abstains or it has $k$ or $k'$ as guru in $d$, i.e., $\GuruOf(i,d) \in \{0,k,k'\}$. Let us present into details algorithms for solving optimization problems, based on this remark.

\begin{lemma} \label{th:sp3}
For SP preference profiles, the problem \textbf{MINDIS} of finding a Nash-stable delegation function $d$ minimizing $\sum_{i\in \VoterSet} (\Score(i,d)-1)$ is solvable in $O(n^3)$ time.
\end{lemma}

\begin{proof}
Build the auxiliary digraph $\Gaux$. For every arc $(k,k') \in \Aaux$, compute an arc-weight $w_{k,k'}$ as follows.
Let $i \in \{ k,  \dots, k'-1\}$, $i \not= s$ (where by abuse of notation $t-1$ is $n$). Define a value $r_{k,k'}(i)$ by: if $i \not=k$, $r_{k,k'}(i)$ is the rank of the preferred guru of $i$ in $\{0,k,k'\} \setminus \{s,t\}$ in $i$'s preference list; otherwise, $r_{k,k'}(i)$ is the rank of $i$ in $i$'s preference list. By Observation~\ref{obs:affectGuruSP}, if $d$ is an equilibrium in which $k$ and $k'$ are two successive gurus, then for every voter $i \in \{ k, \dots, k'-1\}$, it holds that $\Score(i,d) = r_{k,k'}(i)$.
Define now $w_{k,k'} = \sum_{i \in \{k, \dots, k'-1\},\ i \not= s} r_{k,k'}(i)$. All weights $w$ can be computed in $O(n^3)$ time.

Then the total dissatisfaction associated with a Nash-stable delegation function $d$ is equal to the weight of the $s-t$ path associated with its set of gurus $\Gurus(d)$. An optimal solution to the problem \textbf{MINDIS} can then be computed by computing a shortest $s-t$ path for the weights $w$.%\qed
\end{proof}

\begin{lemma} \label{th:sp4}
For SP preference profiles, the problem \textbf{MINMAXVP} of finding a Nash-stable delegation function $d$ minimizing $\max_{i\in \Gurus(d)} \VotingPower(i,d)$ is solvable in $O(n^3)$ time.
\end{lemma}

\begin{proof}
For every arc $(i,j)$ of the auxiliary digraph $\Gaux$, let $w^i_{ij}$ (resp. $w^j_{ij}$) denote the number of voters in $\{i+1, \dots, j-1\}$ whose preferred guru in $\{0,i,j\}$ is $i$ (resp. $j$). Given an $s-t$ path involving a guru $j$, the voting power of guru $j$ is exactly $w^j_{ij} + w^j_{jk}+1$, where $(i,j)$ and $(j,k)$ are the arcs of the $s-t$ path containing $j$. Hence, the problem \textbf{MINMAXVP} is equivalent to the problem of finding an $s-t$ path in the auxiliary digraph that minimizes the maximum value of $w^j_{ij} + w^j_{jk} + 1$ over pairs of consecutive arcs $(i,j)$ and $(j,k)$ of the path.

For every $i \in \VoterSet$, $w \in \{0, \dots, n\}$, let $M(i,w)$ be the minimum value $W$ such that: there exists an $s-i$ path where all gurus in $\{1,\ldots,i\}$ have voting power at most $W$ and $i$ has voting power at most $w$ on her left. 
We prove that all values $M(i,w)$ can be computed in $O(n^3)$ time.
Let $j \in \VoterSet$. Assume all values have been computed for agents on the left of $j$. For each predecessor $i$ of $j$ in $\Gaux$, for each value $w \in \{0, \dots, n\}$, form $M = \max\{M(i,w); w + w_{ij}^i+1\}$, then update $M(j, w^j_{ij}) := \min\{ M(j, w^j_{ij}); M\}$. 
Hence all values $M(j, \cdot)$ can be computed in quadratic time for every $j$. 
Finally the value $M(t,0)$ is the optimal value of the \textbf{MINMAXVP} optimization problem.
The auxiliary digraph $\Gaux$ can be computed in $O(n^2)$ time, and the $w^i_{ij}$ can be computed in $O(n^3)$. Thus the overall complexity is $O(n^3)$. Note that the optimal solution can be obtained by standard bookkeeping techniques without increasing the complexity of the method.%\qed
\end{proof}

\begin{lemma} \label{th:sp5}
For SP preference profiles, the problem \textbf{MINABST} of finding a Nash-stable delegation function $d$ minimizing $|\{i\in \VoterSet | \GuruOf(i,d) = 0\}|$ is solvable in $O(n^3)$ time.
\end{lemma}

\begin{proof}
For every arc $(k,k')$ in the auxiliary digraph $\Gaux$, define $a_{kk'}$ as the number of voters in $\{k+1, \dots, k'-1\}$ whose preferred guru in $\{0, k,k'\}$ is $0$.
By Observation~\ref{obs:affectGuruSP}, if $d$ is an equilibrium in which $k$ and $k'$ are two successive gurus, then for every voter $i \in \{ k+1, \dots, k'-1\}$, the guru $\GuruOf(i,d)$ of $i$ is the most preferred guru of $i$ in $\{0,k,k'\}$, and $i$ abstains if and only if $0$ is her preferred guru in $\{0,k,k'\}$. Hence, the number of voters of $\{k+1, \dots, k'-1\}$ who abstain in $d$ is exactly $a_{kk'}$. Thus for any equilibrium $d$, associated with an $s-t$ path in the auxiliary digraph, the total number of voters who abstain in $d$ is equal to the sum of arc-weights $a$ over the path. Hence an optimal solution to the problem \textbf{MINABST} can then be computed by searching for a shortest $s-t$ path for the weights $a$.
All arc-weights $a$ can be computed in $O(n^3)$ time. Hence an optimal solution to \textbf{MINABST} can be computed in $O(n^3)$ time.%\qed
\end{proof}

\section{Appendix of Section~\ref{sec:symmetrical}}

\subsection{Proof of Observation~\ref{obsguc}}

{\bf Observation~\ref{obsguc}} {\it $G_{U,C}$ has a kernel containing no clause vertex if and only if $(U,C)$ is satisfiable.}

%\OBSGUC
\begin{proof}
If $(U,C)$ has a satisfying assignment, then consider in $G_{U,C}$ the set of variable vertices corresponding to true literals. This set is clearly independent, and absorbing since every clause is satisfied by the assignment. Conversely, a kernel containing no clause vertex must contain exactly one variable vertex  among $v^x_{it}$ and $v^x_{if}$ (for each $i$). Since the set is absorbing, the literals corresponding to this kernel satisfy all the clauses. %\qed
\end{proof}

\subsection{Proof of Theorem~\ref{th:hardnesssymmetric}}

{\bf Theorem \ref{th:hardnesssymmetric}}
{\it Given a symmetrical preference profile $P$:\\ 
-- it is NP-hard to decide whether there exists an equilibrium where no voter abstains, or not. Thus, in particular, {\bf MINABST} is NP-hard.\\
-- {\bf MINDIS} is NP-hard  even if there are no abstainers.\\
-- {\bf MINMAXVP} is NP-hard  even if there are no abstainers.\\
}

The first item has been proven in the main body of the article. We now prove the two other items. For the sake of readability, we prove each of them in a separate lemma (lemmas~\ref{lemma:mindishardnesssymmetric} and \ref{lemma:minmaxvphardness}).

\begin{lemma}\label{lemma:mindishardnesssymmetric}
{\bf MINDIS} is NP-hard in the case of  symmetrical preference profiles, even if there are no abstainers.
\end{lemma}

\begin{proof} 
Let us consider a 3-SAT instance with a set $U$ of variables and a set $C$ of clauses. We will create a preference profile the delegation acceptability digraph of which is made of:
\begin{itemize}
	\item The digraph $G_{U,C}$ associated to $(U,C)$;
    \item A clique $\{v^*,v^*_1,\ldots,v^*_{k-1}\}$ of $k$ vertices (the value of $k$ will be given later). Every vertex of the clique is adjacent to every clause vertex of $G_{U,C}$.
\end{itemize}

Thus, we build a profile on $2n_u+n_c+k$ voters: $2n_u$ `variable voters', $n_c$ `clause voters', and $k$ `clique voters'. Every voter prefers to vote than to abstain (abstention will be the last preferred option for all voters). 
Each agent in $\{v^*_1,\ldots,v^*_{k-1}\}$ have $v^*$ as her first choice. Then the preferences of voters in $\{v^*_1,\ldots,v^*_{k-1}\}$ form a Latin-square, i.e., voter $v^*_i$'s second choice is $v^*_{(i \mod k-1)+1}$, third choice is $v^*_{(i+1 \mod k-1)+1}$ and so on until $v^*_{(i+k-3 \mod k-1)+1}$ is reached. Then agent $v^*_i$ prefers to delegate to voters $v^c_i$, $i\in\{1,\ldots,n_c\}$, then she prefers to vote. Agent $v^*$ prefers to delegate to voters in $\{v^*_1,\ldots,v^*_{k-1}\}$, then she prefers to delegate for voters in $\{v^c_i|i=1,\ldots,n_c\}$, then she prefers to vote.

For example, if we assume $k=4$ and the previous 3-SAT instance ($U= \{x_1,x_2,x_3,x_4,x_5\}$ and $C = \{ (x_1 \lor x_2 \lor \lnot x_3), (\lnot x_2 \lor \lnot x_4\lor x_1), (\lnot x_1 \lor x_3\lor x_5)\}$), then possible preferences could be:
\begin{align*}
v^*  &: v^*_1 \succ_{v^*} v^*_2 \succ_{v^*} v^*_3 \succ_{v^*} v^c_1     \succ_{v^*} v^c_2 \succ_{v^*} v^c_3 \succ_{v^*}  v^* \\ 
v^*_1&: v^* \succ_{v^*_1} v^*_2 \succ_{v^*_1} v^*_3 \succ_{v^*_1}  v^c_1 \succ_{v^*_1} v^c_2 \succ_{v^*_1} v^c_3 \succ_{v^*_1} v^*_1 \\ 
v^*_2&: v^* \succ_{v^*_2} v^*_3 \succ_{v^*_2} v^*_1 \succ_{v^*_2}  v^c_1 \succ_{v^*_2} v^c_2 \succ_{v^*_2} v^c_3 \succ_{v^*_2} v^*_2 \\ 
v^*_3&: v^* \succ_{v^*_3} v^*_1 \succ_{v^*_3} v^*_2 \succ_{v^*_3}  v^c_1 \succ_{v^*_3} v^c_2 \succ_{v^*_3} v^c_3 \succ_{v^*_3} v^*_3
\end{align*} 

Every agent in $\{v^c_i|i=1,\ldots,n_c\}$ first prefers to delegate to the 3 voters in $\{v^x_{it}, v^x_{if}| n= 1, \ldots, n_u\}$ corresponding to the literals of their clause. Then, they prefer to delegate to the voters in $\{v^*_1,\ldots,v^*_{k-1}\}$. Then they prefer to delegate to $v^*$ and then they prefer to vote. For example, if we assume $k=4$ then possible preferences  w.r.t. the previous 3-SAT instance  could be:
\begin{align*}
v^c_1  &: v^x_{1t} \succ_{v^c_1} v^x_{2t} \succ_{v^c_1} v^x_{3f} \succ_{v^c_1} v^*_1 \succ_{v^c_1} v^*_2 \succ_{v^c_1} v^*_3  \succ_{v^c_1} v^* \succ_{v^c_1}  v^c_1 \\ 
v^c_2  &: v^x_{2f} \succ_{v^c_2} v^x_{4f} \succ_{v^c_2} v^x_{1t} \succ_{v^c_2} v^*_1 \succ_{v^c_2} v^*_2 \succ_{v^c_2} v^*_3  \succ_{v^c_2} v^* \succ_{v^c_2}  v^c_2 \\ 
v^c_3  &: v^x_{1f} \succ_{v^c_3} v^x_{3t} \succ_{v^c_3} v^x_{5t} \succ_{v^c_3} v^*_1 \succ_{v^c_3} v^*_2 \succ_{v^c_3} v^*_3  \succ_{v^c_3} v^* \succ_{v^c_3}  v^c_3  
\end{align*}

Lastly, each agent $v^x_{it}$ (resp.  $v^x_{if}$) first prefers to delegate to $v^x_{if}$ (resp.  $v^x_{it}$), then to delegate to voters in $\{v^c_i|i=1,\ldots,n_c\}$ corresponding to clauses that include variable $x_i$ (resp. the negation of variable $x_i$), then they prefer to vote directly. For example, possible preferences  w.r.t. the previous 3-SAT instance could be:
\begin{align*}
v^x_{1t} &: v^x_{1f} \succ_{v^x_{1t}} v^c_1 \succ_{v^x_{1t}} v^c_2 \succ_{v^x_{1t}} v^x_{1t} \\
v^x_{2t} &: v^x_{2f} \succ_{v^x_{2t}} v^c_1 \succ_{v^x_{2t}} v^x_{2t} \\
v^x_{3t} &: v^x_{3f} \succ_{v^x_{3t}} v^c_3 \succ_{v^x_{3t}} v^x_{3t} \\
v^x_{4t} &: v^x_{4f} \succ_{v^x_{4t}} v^x_{4t} \\
v^x_{5t} &: v^x_{5f} \succ_{v^x_{5t}} v^c_3 \succ_{v^x_{5t}} v^x_{5t} \\
v^x_{1f} &: v^x_{1t} \succ_{v^x_{1f}} v^c_3 \succ_{v^x_{1f}} v^x_{1f} \\
v^x_{2f} &: v^x_{2t} \succ_{v^x_{2f}} v^c_2 \succ_{v^x_{2f}} v^x_{2f}\\
v^x_{3f} &: v^x_{3t} \succ_{v^x_{3f}} v^c_1 \succ_{v^x_{3f}} v^x_{3f} \\
v^x_{4f} &: v^x_{4t} \succ_{v^x_{4f}} v^c_2 \succ_{v^x_{4f}} v^x_{4f}\\
v^x_{5f} &: v^x_{5t} \succ_{v^x_{5f}} v^x_{5f}
\end{align*}
We fix $k=3n_c+n_u+n_un_c$ and show that the 3-SAT instance is satisfiable if and only if there exists an equilibrium  with dissatisfaction at most $2k$.

Assume first that the 3-SAT instance is satisfiable. Then, $v^*$ plus the $n_u$ variable vertices in $G_{U,C}$ corresponding to true literals form a kernel in the delegation acceptability digraph. Let us consider the corresponding delegation function where: 
\begin{itemize}
	\item the $n_u+1$ voters in the kernel vote. The dissatisfaction of $v^*$ is $(k-1)+n_c$, the one of each voter corresponding to a true literal is at most $1+n_c$.
    \item voters $v^*_i$ delegate to $v^*$, they have dissatisfaction $0$;
    \item clause voters delegate to a variable voter corresponding to a true literal in the clause, thus with a dissatisfaction at most $2$.
    \item (variable) voters corresponding to false literals delegate to the opposite (true) literal and have dissatisfaction 0.
\end{itemize}

Thus, the dissatisfaction of this equilibrium is at most $k-1+n_c+n_u(1+n_c)+2n_c=2k-1$.\\

Conversely, assume that there is an equilibrium with dissatisfaction at most $2k$. If a voter $v^*_s$ ($\neq v^*$) in the clique votes, then no other clique voter votes, and this already induces a dissatisfaction at least $\sum_{i=1}^{k-2}i$
$ =\frac{(k-1)(k-2)}{2}>2k$ (for $k\geq 7$) for the other $k-2$ vertices $v^*_j,j\neq s$. Thus this is not possible. Similarly, if a clause voter votes, then no voter in the clique can vote, and this already induces a dissatisfaction at least $(k-1)$ for each voter in the clique, thus a global dissatisfaction at least $k(k-1)>2k$ (for $k>4$). 

Then, in the considered equilibrium, $v^*$ votes, no other voter in the clique votes, and no clause voter votes. Then for any $i$ exactly one voter among $v^x_{it}$ and $v^x_{if}$ votes. We conclude the proof by showing that the assignment where a literal is true if the corresponding voter votes is a satisfying assignment. Note that $v^*$ has dissatisfaction $(k-1)+n_c$. If a clause voter delegates to $v^*$, then it has dissatisfaction $k+2$, so the global dissatisfaction is greater than $2k$, impossible. This means that each clause vertex delegates to a voting variable vertex, and thus all the clauses are satisfied by the assignment. % \qed
\end{proof}

\begin{lemma}\label{lemma:minmaxvphardness}
{\bf MINMAXVP} is NP-hard in the case of symmetrical preference profiles, even if there are no abstainers.
\end{lemma}
\begin{proof}
Let us consider a 3-SAT instance with a set $U$ of variables and a set $C$ of clauses. We will create a preference profile, the delegation acceptability (symmetric) digraph of which is made of:
\begin{itemize}
	\item The digraph $G_{U,C}$ associated to $(U,C)$;
    \item A clique $\{v_1,\ldots,v_{n_c+2}\}$ of $n_c+2$ vertices. Every vertex of the clique is adjacent to every clause vertex of $G_{U,C}$.
    \item An independent set $\{v'_1,\ldots,v'_{n_c+2}\}$ of $n_c+2$ vertices. Every vertex $v'_i$ is adjacent to $v_i$.
\end{itemize}

Figure~\ref{MMGRed} illustrates the digraph corresponding to the following 3-SAT instance: 
\begin{align}
U &= \{x_1,x_2,x_3,x_4,x_5\} \label{form3a}\\
C &= \{ (x_1 \lor x_2 \lor \lnot x_3), (\lnot x_2 \lor \lnot x_4\lor x_1), (\lnot x_1 \lor x_3\lor x_5)\} \label{form3b}
\end{align}

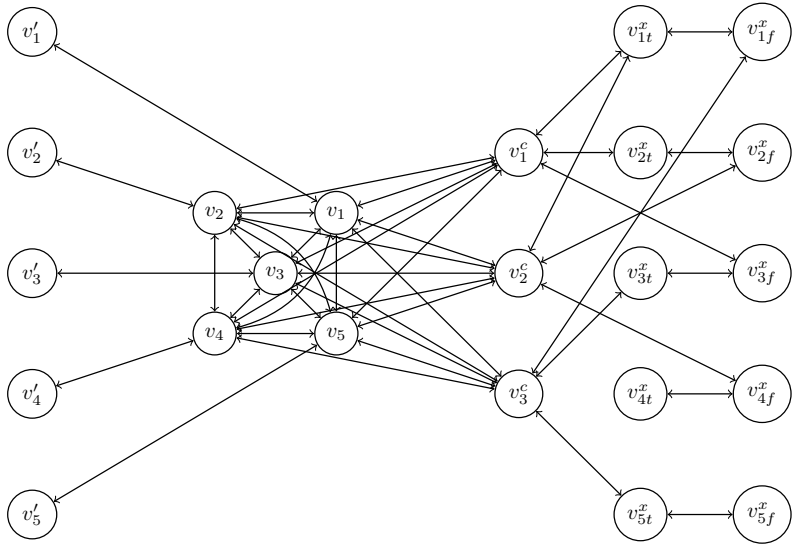
\begin{figure}[!h] 
   \centering
    \scalebox{0.8}{\begin{tikzpicture}[<->,auto,node distance=3cm,semithick]
  \node[circle,draw,text=black] (AT)   at (6,4)                 {$v^x_{1t}$};
  \node[circle,draw,text=black] (BT)   at (6,2)                 {$v^x_{2t}$};
  \node[circle,draw,text=black] (CT)   at (6, 0)                 {$v^x_{3t}$};
  \node[circle,draw,text=black] (DT)   at (6, -2)                {$v^x_{4t}$};
  \node[circle,draw,text=black] (ET)   at (6, -4)                {$v^x_{5t}$};
  \node[circle,draw,text=black] (AF)   at (8,4)                 {$v^x_{1f}$};
  \node[circle,draw,text=black] (BF)   at (8,2)                 {$v^x_{2f}$};
  \node[circle,draw,text=black] (CF)   at (8, 0)                 {$v^x_{3f}$};
  \node[circle,draw,text=black] (DF)   at (8, -2)                {$v^x_{4f}$};
  \node[circle,draw,text=black] (EF)   at (8, -4)                {$v^x_{5f}$};
  \node[circle,draw,text=black] (C1)   at (4,2)                 {$v^c_1$};
  \node[circle,draw,text=black] (C2)   at (4,0)                 {$v^c_2$};
  \node[circle,draw,text=black] (C3)   at (4,-2)                 {$v^c_3$};
  
  \node[circle,draw,text=black] (TC1)   at (1,1)                 {$v_1$};
  \node[circle,draw,text=black] (TC5)   at (1,-1)                 {$v_5$};
  \node[circle,draw,text=black] (TC2)   at (-1,1)                {$v_2$};
  \node[circle,draw,text=black] (TC4)   at (-1,-1)                {$v_4$};
  \node[circle,draw,text=black] (TC3)   at (0,0)                {$v_3$};
  
  \node[circle,draw,text=black] (TL1)   at (-4,4)                 {$v_1'$};
  \node[circle,draw,text=black] (TL2)   at (-4,2)                 {$v_2'$};
  \node[circle,draw,text=black] (TL3)   at (-4,0)                {$v_3'$};
  \node[circle,draw,text=black] (TL4)   at (-4,-2)                {$v_4'$};
  \node[circle,draw,text=black] (TL5)   at (-4,-4)                {$v_5'$};
  
  \path
 	    (TC1) edge  node {} (C1)
   	    (TC1) edge  node {} (C2)
   	    (TC1) edge  node {} (C3)
        
        (TC2) edge  node {} (C1)
   	    (TC2) edge  node {} (C2)
   	    (TC2) edge  node {} (C3)
        
        (TC3) edge  node {} (C1)
   	    (TC3) edge  node {} (C2)
   	    (TC3) edge  node {} (C3)
        
        (TC4) edge  node {} (C1)
   	    (TC4) edge  node {} (C2)
   	    (TC4) edge  node {} (C3)
        
        (TC5) edge  node {} (C1)
   	    (TC5) edge  node {} (C2)
   	    (TC5) edge  node {} (C3)
        
        (TC1) edge  node {} (TC2)
        (TC1) edge  node {} (TC3)
        (TC1) edge[bend left]  node {} (TC4)
        (TC1) edge  node {} (TC5)
        
        (TC2) edge  node {} (TC3)
        (TC2) edge  node {} (TC4)
        (TC2) edge[bend left]  node {} (TC5)
        
        (TC3) edge  node {} (TC4)
        (TC3) edge  node {} (TC5)
        
        (TC4) edge  node {} (TC5)
        
        (TL1) edge  node {} (TC1)
        (TL2) edge  node {} (TC2)
        (TL3) edge  node {} (TC3)
        (TL4) edge  node {} (TC4)
        (TL5) edge  node {} (TC5)
        
  	    (C1) edge  node {} (AT)
        (C1) edge  node {} (BT)
        (C1) edge  node {} (CF)

  	    (C2) edge  node {} (BF)
        (C2) edge  node {} (DF)
        (C2) edge  node {} (AT)

  	    (C3) edge  node {} (AF)
        (C3) edge  node {} (CT)
        (C3) edge  node {} (ET)

        (AT) edge node {} (AF)
        (BT) edge node {} (BF)
        (CT) edge node {} (CF)
        (DT) edge node {} (DF)
        (ET) edge node {} (EF);
	  \end{tikzpicture}}
      \caption{Delegation acceptability digraph.}
     \label{MMGRed}
\end{figure}

Thus we build a profile on $2n_u+3n_c+4$ voters, all of them prefer to vote than to abstain. 

We first precise the preferences of voters in $\{v_1, \ldots, v_{n_c+2}\}$. Each voter in $\{v_1, \ldots, v_{n_c+2}\}$ prefers first to delegate to voters in $\{v_j^c|j=1,\ldots, n_c\}$ and in the same order, $v_1^c$ first, $v_2^c$ second and so on. Then, they prefer to delegate to the other voters in $\{v_1, \ldots, v_{n_c+2}\}$, then to the corresponding $v_{j}'$ and lastly, they prefer to vote.
For example, in the instance described by Equations \ref{form3a} and \ref{form3b}, possible preferences are given by:
\begin{align*}
v_1  &: v^c_1 \succ_{v_1} v^c_2 \succ_{v_1} v^c_3 \succ_{v_1} v_2     \succ_{v_1} v_3 \succ_{v_1} v_4 \succ_{v_1} v_5 \succ_{v_1}  v_1' \succ_{v_1} v_1 \\ 
v_2  &: v^c_1 \succ_{v_2} v^c_2 \succ_{v_2} v^c_3 \succ_{v_2} v_1 \succ_{v_2} v_3 \succ_{v_2} v_4 \succ_{v_2} v_5 \succ_{v_2} v_2' \succ_{v_2} v_2 \\ 
v_3  &: v^c_1 \succ_{v_3} v^c_2 \succ_{v_3} v^c_3 \succ_{v_3} v_1 \succ_{v_3} v_2 \succ_{v_3} v_4 \succ_{v_3} v_5 \succ_{v_3} v_3' \succ_{v_3} v_3 \\ 
v_4  &: v^c_1 \succ_{v_4} v^c_2 \succ_{v_4} v^c_3 \succ_{v_4} v_1 \succ_{v_4} v_2 \succ_{v_4} v_3 \succ_{v_4} v_5 \succ_{v_4} v_4' \succ_{v_4} v_4\\
v_5  &: v^c_1 \succ_{v_5} v^c_2 \succ_{v_5} v^c_3 \succ_{v_5} v_1 \succ_{v_5} v_2 \succ_{v_5} v_3 \succ_{v_5} v_4 \succ_{v_5} v_5' \succ_{v_5} v_5 
\end{align*} 
Without defining further the other preferences, we will show the following result: the 3-SAT instance is satisfiable iff there exists a Nash-stable delegation function in which each guru has a voting power which is strictly less than $n_c+3$. Indeed, note that if some voters in $\{v_j^c|j=1,\ldots, n_c\}$ are gurus, then one of them will be endorsed by all voters in $\{v_1, \ldots, v_{n_c+2}\}$ and will thus have a power of at least $n_c+3$. Additionally, if one voter in $\{v_1, \ldots, v_{n_c+2}\}$ is a guru, then she will collect the votes of all other voters in $\{v_1, \ldots, v_{n_c+2}\}$ and the vote of the corresponding $v_i'$ and will thus have a power of at least $n_c+3$. Hence, a Nash-stable delegation function in which each guru has a power which is strictly less than $n_c+3$ corresponds to a Nash-stable delegation function in which no voters in $\{v_1, \ldots, v_{n_c+2}\}\cup\{v_j^c|j=1,\ldots, n_c\}$ are gurus. This is possible if all voters in $\{v_1',\ldots, v_{n_c+2}'\}$ are gurus and if each voter in $\{v_j^c|j=1,\ldots, n_c\}$ delegates to a guru in $\{v_{it}^x, v_{if}^x|i=1,\ldots, n_u\}$. In this case, the power of a guru is at most $n_c+2$ and the gurus in $\{v_{it}^x, v_{if}^x|i=1,\ldots, n_u\}$ form a truth assignment that satisfies all clauses.%\qed

\end{proof}

\subsection{Proof of Theorem~\ref{th:cvbrd}}

%\THCVBRD

%\begin{restatable*}{theorem}{THCVBRD}%\label{th:cvbrd}
%Given a symmetric preference profile $P$, a BRD dynamics always converges in at most 3 rounds.
%\end{restatable*}

{\bf Theorem~\ref{th:cvbrd}.}
{\it Given a symmetric preference profile $P$, a BRD dynamics always converges in at most 3 rounds.
}

\begin{proof}
Let us consider $j=T(t)$ for some step $t$. Suppose that $j$ decides to vote herself when she has the token at time $t$. Then, for any $t'\geq t$, $j$ remains a guru. Indeed, since she decides to vote at time $t$, it means that no voter in $\Acc(j)$ were gurus. While $j$ is a guru, then a voter $i\in \Acc(j)$ cannot become a guru under BRD since by symmetry $j\in \Acc(i)$. 

For a voter $j$, consider a time $t$ in the second round where she receives the token. 
\begin{itemize}
\item if $j$ decides to vote, from the previous argument she will be a guru forever after time $t$.
\item if $j$ delegates (directly or indirectly) to a guru $i$, since we are in the second round $i$  already had the token, already decided to be a guru, and thus by the previous argument will remain a guru forever. Then, $j$ will never be a guru ($i\in \Acc(j)$ will always be available as a guru).
\item If $j$ abstains, then she prefers to abstain than to vote so she will never be a guru.
\end{itemize}
This means that after the second round the set of gurus is fixed. Then, in the third round, gurus remain gurus, and non gurus choose their most preferred guru in the set of gurus, or abstain (if they prefer to abstain than to vote or to delegate to a guru). Thus we reach an equilibrium at the end of the third round.\\

Note that 3 rounds are necessary. Consider for instance a profile with 3 voters, 1 prefers 2 and then to vote, 3 prefers 2 and then to vote, and 2 prefers 1, then 3, and then to vote.
We give the token to 1,2,3 (first round), 2,3,1 (second round), 1, 3, 2 (third round).
Under BRD we get $d_1(1)=2$, $d_2(2)=3$, $d_3(3)=3$ (end of the first round, 3 is a guru), $d_4(2)=3$, $d_5(3)=3$, $d_6(1)=1$ (end of the second round, the set of gurus \{1,3\} is definitive), $d_7(1)=1$, $d_8(3)=3$, $d_9(2)=1$ (end of the third round, convergence).% \qed
\end{proof}

\section{Appendix of Section~\ref{sec:dbsn}}
\label{app:db}
\subsection{Example of Distance-Based profile}

\begin{example} \label{ex:DB}
Consider 5 voters that are points in the following 2-dimentional grid ($2$ is at coordinates $(0,0)$, $1$ is at coordinate $(0,1)$, $4$ at coordinates $(2,0)$,\dots). 

%UGGLY picture: put the grid, and possibly a circle on one vertex to represent $\At_i$.

\begin{figure}[!h] 
   \centering
    \scalebox{1}{\begin{tikzpicture}[->,>=stealth',shorten >=1pt,auto,node distance=3cm,semithick]

  \node[circle,draw,text=black] (A)   at (0,2)                 {$1$};
  \node[circle,draw,text=black] (B)   at (0,0)                 {$2$};
  \node[circle,draw,text=black] (C)   at (2,0)                 {$3$};
  \node[circle,draw,text=black] (D)   at (4,0)                 {$4$};
  \node[circle,draw,text=black] (E)   at (2,2)                 {$5$};
  \node (F) at (4,2) {};
  %\node (G) at (4,2) {};
  
  \draw[dashed,-] (E) -- (6,2);
  \draw[dashed,->] (D) -- (6,0);
  \draw[dashed,-] (D) -- (4,3);
  \draw[dashed,->] (A) -- (0,3);
  \draw[dashed,-] (E) -- (2,3);
  
  %\draw (A) <-> (B);
  \path (A) edge[dashed,-] node {} (B)
        (A) edge[dashed,-]  node {} (E)
  
  		(B) edge[dashed,-] node {} (C)
  	    (C) edge[dashed,-] node {} (E)
        
  	    (C) edge[dashed,-]  node {} (D);
        %(E) edge[dashed]  node {} (F)
        %(D) edge[dashed]  node {} (F);
	  \end{tikzpicture}}
    \caption{.}
    \label{grid}
\end{figure}
The distance is the Euclidean distance, and the acceptability thresholds are: $\At{1}=2$ (so $1$ accepts 2, 3 and 5 as a guru, but not 4), $\At{2}=1.5$, $\At{3}=2$, $\At{4}=1$ and $\At{5}=1$. A partial preference profile which is DB w.r.t. the Euclidean distance and the acceptability thresholds of the voters is given in Figure \ref{exDBSNAccpt}, together with its delegation-acceptability digraph (note that the delegation acceptability digraph is unique given the distance and the acceptability thresholds of the voters).
 %\MessageFromBruno{Is it worth giving an example of preferences fitting the spatial representation? Or even giving the example? It takes a lot of space (appendix?)}: 

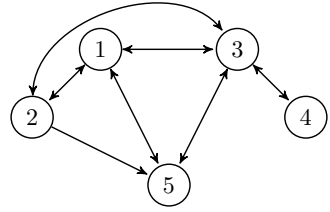
\begin{figure}
\begin{minipage}[c]{.2\linewidth}
\begin{align*}
1&:2\succ_1 3 \succ_1 5 \succ_1 1 \\%\succ_1 0 \succ_1 4  \\ 
2&:1\succ_2 5  \succ_2 3 \succ_2 0 \\% \succ_1 2  \succ_2 4  \\ 
3&:4\succ_3 1 \succ_3 5 \succ_3 2 \succ_3 3 \\%\succ_3 0 \\ 
4&:3\succ_4 4 \\%\succ_4 1 \succ_4 5 \succ_4 0 \succ_4 2   \\ 
5&:1\succ_5 3 \succ_5 5 \\%\succ_5 4 \succ_5 2 \succ_5 0
\end{align*}
\end{minipage}\hfill
\begin{minipage}[c]{.65\linewidth}

%\begin{figure} 
   \centering
    \scalebox{0.9}{\begin{tikzpicture}[->,>=stealth',shorten >=1pt,auto,node distance=3cm,semithick]

  \node[circle,draw,text=black] (A)   at (0,2)                 {$1$};
  \node[circle,draw,text=black] (B)   at (-1,1)                 {$2$};
  \node[circle,draw,text=black] (C)   at (2,2)                 {$3$};
  \node[circle,draw,text=black] (D)   at (3,1)                 {$4$};
  \node[circle,draw,text=black] (E)   at (1, 0)                 {$5$};
  
 % \draw[dotted]  (A) -- (B);
  \path (A) edge[ <->]  node {} (B)
        (A) edge[<->]  node {} (C)
        (A) edge[<->]  node {} (E)
  
  		(B) edge[<->, bend left=70] node {} (C)
  	    (B) edge[] node {} (E)
        
  	    (C) edge[<->]  node {} (D)
        (C) edge[<->]  node {} (E);
	  \end{tikzpicture}}
    
\end{minipage}

\caption{The delegation-acceptability digraph induced by the Euclidean distance on voters in Figure \ref{grid} and the acceptability thresholds of the voters.}
    \label{exDBSNAccpt}

\end{figure}

\subsection{Proof of Theorem~\ref{th:membharddb}}

{\bf Theorem \ref{th:membharddb}.}
{\it {\bf MEMB} is NP-hard in the case of DB preference profiles, even if there are no abstainers.}

\begin{proof} 
We build a reduction where the voters are vertices of a graph, and the distance between voters $i$ and $j$ is the shortest path (number of edges, the graph is unweighted) between the two vertices representing $i$ and $j$ in the graph. 

Let us consider a 3-SAT instance with a set $U$ of variables and a set $C$ of clauses. We consider a graph made of:
\begin{itemize}%[topsep = 0 pt]
	\item The undirected version of the graph $G_{U,C}$ associated to $(U,C)$ (see Figure~\ref{graphguc} in Section~\ref{sec:symmetrical});
    \item Two adjacent vertices $v_t$ and $v_q$; $v_t$ is also adjacent to all clause vertices $v^c_j$.
\end{itemize}

Thus we have $2n_u+n_c+2$ voters. As we said, we define $dist(i,j)$ as the shortest path between $i$ and $j$ in the graph. The acceptability threshold is 1 for all voters except $v_q$ which has an acceptability threshold of 2; they all prefer to vote than to abstain.

Let us show that the 3-SAT instance is satisfiable iff the DB preference profile induced by the corresponding distance admits a Nash-stable delegation function in which $v_q$ is a guru. 

Assume that there exists a Nash-stable delegation function $d$ in which $v_q$ is a guru. As $v_q$ is a guru, then voters $v_t$ and $v^c_j, \forall j \in \{1,\ldots,n_c\}$ cannot be gurus as they are at a distance of less than 2 from $v_q$. Contrarily to voter $v_t$ who accepts to delegate to $v_q$, each voter $v^c\in \{v^c_j|j=1,\ldots,n_c\}$ will necessarily delegate to one of the three voters corresponding to the literals of its clause. Lastly, note that as $v^x_{it}$ and $v^x_{if}$ are connected for all $i \in \{1,\ldots n_u\}$, then at most one of them can be a guru in a Nash stable delegation function. Furthermore, as the voters $v^c\in \{v^c_j|j=1,\ldots,n_c\}$ cannot be gurus, one voter out of $\{v^x_{it},v^x_{if}\}$ will have to be a guru for all $i \in \{1,\ldots n_u\}$. Now consider the truth assignment that sets to true a variable $x_i$ iff $v^x_{it}\in \Gurus(d)$. It is easy to check that this truth assignment satisfies each clause in $C$.

Conversely, if there exists a truth assignment $X$ satisfying each clause in $C$ then consider the delegation function $d$ such that $\Gurus(d)$ is composed of $v_q$, all variables $v^x_{it}$ such that $x_i$ is set to true in $X$ and all variables $v^x_{if}$ such that $x_i$ is set to false in $X$. For all voters $j$ not in $\Gurus(d)$, $d(j)$ is then given by the voter that $j$ prefers in $\Gurus(d)$. It is easy to see that $d$ is Nash-stable. % \qed
\end{proof}

\subsection{Proof of Theorem~\ref{th:bdrconvdb}}

{\bf Theorem~\ref{th:bdrconvdb}.}
{\it Given a DB preference profile $P$, a BRD dynamics always converges.
}

\begin{proof}
Let us recall that we assume that each voter has the token infinitely many times. Consider a DB preference profile $P$, and a BRD dynamics with a starting delegation $d_0$ and a token function $T$. We assume that voters are numbered $1,2\dots,n$ in such a way that $\At{i}\leq \At{i+1}$, $i=1,\dots,n-1$. 

Let us define $G$ as the set of voters which are gurus (vote) infinitely many times in the dynamics: $G=\{i_1,\dots,i_s\}$ with $i_1\leq i_2\leq \dots \leq i_s$. Note that, obviously, $G$ contains no abstainers. Since voters in $V\setminus G$ are gurus finitely many times, let us consider a step $t_0$ such that, for any $t\geq t_0$, no voter in $V\setminus G$ are gurus (they always delegate or abstain). 

For $k\in \{1,2,\dots,s\}$, let us now define $t_k$ as the first time $t>t_{k-1}$ such that $i_k$ has the token and decides to vote (thus it is a guru).

We now show by recurrence on $k$ that for any $k\in \{1,2,\dots,s\}$, any $t\geq t_k$, $i_k$ is a guru at time $t$ ($i_k$ remains a guru forever after time $t_k$).

Consider $k=1$. At $t_1$, $i_1$ decides to vote. This means that no voter in $\Acc(i_1)$ is a guru. Then, while $i_1$ is a guru: 
\begin{itemize}%[topsep=0pt]
\item no voter $j>i_1$ in $\Acc(i_1)$ ever becomes a guru: indeed, since $\At{i}$ are in non decreasing order, if $j\in \Acc(i_1)$ then $i_1\in \Acc(j)$. While $i_1$ is a guru $j$ does not decide to vote.
\item no voter $j<i_1$ in $\Acc(i_1)$ ever becomes a guru: indeed, these are in $V\setminus G$ and since $t_1\geq t_0$ we know that they always delegate or abstain.
\end{itemize}
Then no voter in $\Acc(i_1)$ becomes a guru, so $i_1$ will vote (be a guru) forever.\\
The inductive step is almost similar. Suppose that the claim is true up to $k-1$, and consider step $k$. At $t_k$, $i_k$ decides to vote. This means that no voter in $\Acc(i_k)$ is a guru. Then, while $i_k$ is a guru: 
\begin{itemize}%[topsep=0pt]
\item no voter $j>i_k$ in $\Acc(i_k)$ ever becomes a guru, for the same reason as previously.
\item no voter $j<i_k$ in $\Acc(i_k)$ ever becomes a guru: if $j\in V\setminus G$ this follows as previously from the fact that $t_k\geq t_0$. If $j\in G$, $j<i_k$ implies that $j=i_f$ for some $f<k$. Then, by induction, $j$ is a guru forever from step $t_{f}<t_k$, thus she cannot be in $\Acc(i_k)$ (since at time $t_k$ $i_k$ decides to vote).   
\end{itemize}

Thus, at time $t_{s}$: voters in $G$ are gurus forever, and voters in $V\setminus G$ never become gurus. From $t_s$ we only have to wait for another round that each voter has the token one more time. Gurus will maintain their choice, while non gurus will choose (thanks to BRD) their most preferred guru in $G$, or abstain (if they prefer to abstain than to vote or to delegate to gurus). We reach a Nash-stable delegation function. %\qed
\end{proof}

\end{example}

\end{document}